\documentclass[article]{IEEEtran} \newcommand{\figsize}{3.5}

\usepackage{graphicx,psfrag}
\usepackage[cmex10]{amsmath}
\usepackage{amsfonts,amssymb,amsthm,latexsym,hyperref}
\usepackage{algorithm,algorithmic}

\usepackage{cite}
\usepackage[usenames,dvipsnames]{color}
\usepackage[normalem]{ulem} 

\usepackage{etoolbox}
\preto\subequations{\ifhmode\unskip\fi}

 \newcommand{\putFrag}[5]{\begin{figure}[t]
                            \centering
                            #4
                            \includegraphics[width=#3in,#5,clip]{figures/#1.eps}
                            \vspace{-2mm}
                            \caption{#2}
                            \label{fig:#1}
                          \end{figure} }
 \newcommand{\putTable}[3]{\begin{table}[t]
                            \centering
                            #3
                            \vspace{2mm}
                            \caption{#2}
                            \vspace{-4mm}
                            \label{tab:#1}
                          \end{table} }
 \newcommand{\capFrag}[2]{}
 \newcommand{\capTable}[2]{}

 \renewcommand{\tilde}{\widetilde}
 \renewcommand{\hat}{\widehat}
 \renewcommand{\bar}{\overline}
 \newcommand{\defn}{\triangleq}

 \newcommand{\hvec}[1]{\ensuremath{\Hat{\boldsymbol{#1}}}}
 
 \renewcommand{\vec}[1]{\ensuremath{\boldsymbol{#1}}}
 \newcommand{\mat}[1]{\ensuremath{\begin{bmatrix}#1\end{bmatrix}}}

 \newcommand{\norm}[1]{\ensuremath{\| #1 \|}}
 \newcommand{\mc}[1]{\ensuremath{\mathcal{#1}}}

 \newcommand{\Real}{{\mathbb{R}}}

 \newcommand{\tran}{^{\top}}
 
 \newcommand*\deriv{\mathop{}\!\mathrm{d}}
 \newcommand{\true}{^0}

 \DeclareMathOperator{\E}{\mathbb{E}}

 \DeclareMathOperator{\prox}{prox}

 \newtheorem{theorem}{Theorem}
 \newtheorem{lemma}{Lemma}

 \renewcommand{\eqref}[1]{(\ref{eq:#1})}
 \newcommand{\Eqref}[1]{Equation~(\ref{eq:#1})}
 \newcommand{\Figref}[1]{Figure~\ref{fig:#1}}
 \newcommand{\figref}[1]{Fig.~\ref{fig:#1}}
 \newcommand{\tabref}[1]{Table~\ref{tab:#1}}
 \newcommand{\secref}[1]{Section~\ref{sec:#1}}

 \newcommand{\lemref}[1]{Lemma~\ref{lem:#1}}
 \newcommand{\thmref}[1]{Theorem~\ref{thm:#1}}

 \newcommand{\algref}[1]{Algorithm~\ref{alg:#1}}
 \newcommand{\lineref}[1]{line~\ref{line:#1}}

 \newcounter{comment}[section]
 
 \newcounter{texthead}[section]

 \newcommand{\map}{_{\textsf{map}}}
 
 \newcommand{\RED}{_{\textsf{red}}}

 \newcommand{\rhoRED}{\rho\RED}
 \newcommand{\CRED}{C\RED}
 \newcommand{\Jf}{J\vec{f}}

 \newcommand{\rhoTR}{\rho_{\text{\sf TR}}}
 \newcommand{\px}{p_{\text{\sf x}}} 
 \newcommand{\pr}{p_{\text{\sf r}}}
 \newcommand{\pxgr}{p_{\text{\sf x$|$r}}} 
 \newcommand{\fmmse}{\vec{f}_{\text{\sf mmse},\nu}}
 \newcommand{\fhatmmse}{\hvec{f}_{\text{\sf mmse},\nu}}
 \newcommand{\ftheta}{\vec{f}_{\vec{\theta}}}
 \newcommand{\fthetahat}{\vec{f}_{\hvec{\theta}}}
 \newcommand{\pxhat}{\hat{\px}}
 \newcommand{\prhat}{\hat{\pr}}
 \newcommand{\pxgrhat}{\hat{\pxgr}}
 \newcommand{\pxsmooth}{\tilde{\px}}
 \newcommand{\eLHb}{e^{\text{\sf LH,1}}_{\vec{f}}}
 \newcommand{\eLHa}{e^{\text{\sf LH,2}}_{\vec{f}}}
 
 \newcommand{\eqLH}{\stackrel{\text{\tiny\sf LH}}{=}}
 
 \newcommand{\pnp}{_{\textsf{pnp}}}
 \newcommand{\admm}{_{\textsf{admm}}}
 \newcommand{\redadmm}{_{\textsf{red-admm}}}
 \newcommand{\redpg}{_{\textsf{red-pg}}}

\begin{document}

\setlength{\arraycolsep}{0.5mm}

\title{Regularization by Denoising: Clarifications and New Interpretations}

\author{Edward T. Reehorst and Philip Schniter, \IEEEmembership{Fellow, IEEE}
        \thanks{E.~T.~Reehorst (email: reehorst.3@osu.edu) 
           and P.~Schniter (email: schniter.1@osu.edu) are with
           the Department of Electrical and Computer Engineering,
           The Ohio State University, 
           Columbus, OH, 43210.
           Their work is supported in part by
           the National Science Foundation under grants
           CCF-1527162 and CCF-1716388
           and the National Institutes of Health under grant
           R01HL135489.}}

\maketitle 

\begin{abstract} 
Regularization by Denoising (RED), as recently proposed by Romano, Elad, and Milanfar, is powerful image-recovery framework that aims to minimize an explicit regularization objective constructed from a plug-in image-denoising function.
Experimental evidence suggests that the RED algorithms are state-of-the-art.
We claim, however, that explicit regularization does not explain the RED algorithms.
In particular, we show that many of the expressions in the paper by Romano et al.\ hold only when the denoiser has a symmetric Jacobian, and we demonstrate that such symmetry does not occur with practical denoisers such as non-local means, BM3D, TNRD, and DnCNN.
To explain the RED algorithms, we propose a new framework called Score-Matching by Denoising (SMD), which aims to match a ``score'' (i.e., the gradient of a log-prior).
We then show tight connections between SMD, kernel density estimation, and constrained minimum mean-squared error denoising.
Furthermore, we interpret the RED algorithms from Romano et al.\ and propose new algorithms with acceleration and convergence guarantees.
Finally, we show that the RED algorithms seek a consensus equilibrium solution, which facilitates a comparison to plug-and-play ADMM.
\end{abstract}

\section{Introduction} \label{sec:intro}
 
Consider the problem of recovering a (vectorized) image $\vec{x}\true\in\Real^N$ from noisy linear measurements $\vec{y}\in\Real^M$ of the form
\begin{align}
\vec{y}=\vec{A}\vec{x}\true+\vec{e} 
\label{eq:yAx},
\end{align} 
where $\vec{A}\in\Real^{M\times N}$ is a known linear transformation and $\vec{e}$ is noise.
This problem is of great importance in many applications and has been studied for several decades. 

One of the most popular approaches to image recovery is the ``variational'' approach, where one poses and solves an optimization problem of the form
\begin{align}
\vec{\hat{x}}=\arg\min_{\vec{x}} \big\{ \ell(\vec{x};\vec{y}) + \lambda\rho(\vec{x}) \big\}
\label{eq:variational}.
\end{align}
In \eqref{variational}, $\ell(\vec{x};\vec{y})$ is a loss function that penalizes mismatch to the measurements, $\rho(\vec{x})$ is a regularization term that penalizes mismatch to the image class of interest, and $\lambda>0$ is a design parameter that trades between loss and regularization.
A prime advantage of the variational approach is that, in many cases, efficient optimization methods can be readily applied to \eqref{variational}.

A key question is: How should one choose the loss $\ell(\cdot;\vec{y})$ and regularization $\rho(\cdot)$ in \eqref{variational}?
As discussed in the sequel, the MAP-Bayesian interpretation suggests that they should be chosen in proportion to the negative log-likelihood and negative log-prior, respectively.
The trouble is that accurate prior models for images are lacking.

Recently, a breakthrough was made by Romano, Elad, and Milanfar in \cite{Romano:JIS:17}.
Leveraging the long history (e.g., \cite{Buades:MMS:05,Milanfar:SPM:13}) and recent advances (e.g., \cite{Chen:TPAMI:17,Zhang:TIP:17}) in image denoising algorithms,
they proposed the \emph{regularization by denoising} (RED) framework,
where an explicit regularizer $\rho(\vec{x})$ is constructed from
an image denoiser $\vec{f}:\Real^N\rightarrow\Real^N$ 
using the simple and elegant rule 
\begin{align}
\rhoRED(\vec{x}) = \frac{1}{2}\vec{x}\tran\big(\vec{x}-\vec{f}(\vec{x})\big)
\label{eq:rhoRED0} .
\end{align}
Based on this framework, they proposed several recovery algorithms (based on steepest descent, ADMM, and fixed-point methods, respectively) that yield state-of-the-art performance in deblurring and super-resolution tasks.

In this paper, we provide some clarifications and new interpretations of the excellent RED algorithms from \cite{Romano:JIS:17}.
Our work was motivated by an interesting empirical observation: With many practical denoisers $\vec{f}(\cdot)$, 
the RED algorithms do not minimize the RED variational objective 
``$\ell(\vec{x};\vec{y})+\lambda\rhoRED(\vec{x})$.''
As we establish in the sequel, the RED regularization \eqref{rhoRED0} is justified only for denoisers with symmetric Jacobians, which unfortunately does not cover many state-of-the-art methods such as
non-local means (NLM) \cite{Buades:CVPR:05}, 
BM3D \cite{Dabov:TIP:07}, 
TNRD \cite{Chen:TPAMI:17}, 
and
DnCNN \cite{Zhang:TIP:17}.
In fact, we are able to establish a stronger result: For non-symmetric denoisers, there exists no regularization $\rho(\cdot)$ that explains the RED algorithms from \cite{Romano:JIS:17}.

In light of these (negative) results, there remains the question of how to explain/understand the RED algorithms from \cite{Romano:JIS:17} when used with non-symmetric denoisers. 
In response, we propose a framework called \emph{score-matching by denoising} (SMD), which aims to match the ``score'' (i.e., the gradient of the log-prior) rather than to design any explicit regularizer. 
We then show tight connections between SMD, kernel density estimation \cite{Parzen:AMS:62}, and constrained minimum mean-squared error (MMSE) denoising. 
In addition, we provide new interpretations of the RED-ADMM and RED-FP algorithms proposed in \cite{Romano:JIS:17}, and we propose novel RED algorithms with faster convergence.
Inspired by \cite{Buzzard:JIS:18}, we show that the RED algorithms seek to satisfy a consensus equilibrium condition that allows a direct comparison to the plug-and-play ADMM algorithms from \cite{Venkatakrishnan:GSIP:13}

The remainder of the paper is organized as follows.
In \secref{back} we provide more background on RED and related algorithms such as plug-and-play ADMM \cite{Venkatakrishnan:GSIP:13}.
In \secref{clarifications}, we discuss the impact of Jacobian symmetry on RED and test whether this property holds in practice. 
In \secref{new}, we propose the SMD framework.
In \secref{algorithms}, we present new interpretations of the RED algorithms from \cite{Romano:JIS:17} and new algorithms based on 
accelerated proximal gradient methods. 
In \secref{CE}, we perform an equilibrium analysis of the RED algorithms,
and, in \secref{conclusion}, we conclude.

\section{Background} \label{sec:back}

\subsection{The MAP-Bayesian Interpretation} \label{sec:Bayesian}

For use in the sequel, we briefly discuss the Bayesian maximum a posteriori (MAP) estimation framework \cite{Bishop:Book:07}. 
The MAP estimate of $\vec{x}$ from $\vec{y}$ is defined as
\begin{align}
\hvec{x}\map = \arg\max_{\vec{x}} p(\vec{x}|\vec{y}) ,
\end{align}
where $p(\vec{x}|\vec{y})$ 
denotes the probability density of $\vec{x}$ given $\vec{y}$.
Notice that, from Bayes rule
$p(\vec{x}|\vec{y}) = p(\vec{y}|\vec{x}) p(\vec{x}) / p(\vec{y})$
and the monotonically increasing nature of $\ln(\cdot)$,
we can write
\begin{align}
\hvec{x}\map 
&= \arg\min_{\vec{x}} \big\{  -\ln p(\vec{y}|\vec{x}) - \ln p(\vec{x}) \big\}
\label{eq:MAP}.
\end{align}

MAP estimation \eqref{MAP} has a direct connection to variational optimization \eqref{variational}:
the log-likelihood term $-\ln p(\vec{y}|\vec{x})$ 
corresponds to the loss $\ell(\vec{x};\vec{y})$ 
and the log-prior term $-\ln p(\vec{x})$ 
corresponds to the regularization $\lambda\rho(\vec{x})$.
For example, with additive white Gaussian noise (AWGN)
$\vec{e}\sim \mc{N}(\vec{0},\sigma_e^2\vec{I})$, 
the log-likelihood implies a quadratic loss:
\begin{align}
\ell(\vec{x};\vec{y})=\frac{1}{2\sigma_e^2}\|\vec{Ax}-\vec{y}\|^2
\label{eq:quadloss}.
\end{align}
Equivalently, the normalized loss 
$\ell(\vec{x};\vec{y})=\frac{1}{2}\|\vec{Ax}-\vec{y}\|^2$
could be used if $\sigma_e^2$ was absorbed into $\lambda$.

\subsection{ADMM} \label{sec:ADMM}
  
A popular approach to solving \eqref{variational} is through ADMM \cite{Boyd:FTML:11}, which we now review.
Using variable splitting, \eqref{variational} becomes
\begin{align}
\vec{\hat{x}} = \arg\min_{\vec{x}} \big\{ \ell(\vec{x};\vec{y})+\lambda\rho(\vec{v}) \big\} ~~ \text{s.t.}\ \vec{x}=\vec{v} 
\label{eq:variable_splitting}.
\end{align}
Using the augmented Lagrangian, problem \eqref{variable_splitting} can be reformulated as  
\begin{align}
\min_{\vec{x},\vec{v}}\max_{\vec{p}} \Big\{
\ell(\vec{x};\vec{y})+\lambda\rho(\vec{v})+\vec{p}\tran(\vec{x}-\vec{v})+\frac{\beta}{2}\norm{\vec{x}-\vec{v}}^2 \Big\}
\label{eq:augmented_lagrangian_nonscaled}
\end{align}
using Lagrange multipliers (or ``dual'' variables) $\vec{p}$ and a design parameter $\beta>0$.
Using $\vec{u}\triangleq \vec{p}/\beta$, \eqref{augmented_lagrangian_nonscaled} can be simplified to
\begin{align}
\min_{\vec{x},\vec{v}}\max_{\vec{u}} \Big\{
\ell(\vec{x};\vec{y})+\lambda\rho(\vec{v})+\frac{\beta}{2}\norm{\vec{x}-\vec{v}+\vec{u}}^2 - \frac{\beta}{2}\|\vec{u}\|^2 \Big\}
\label{eq:augmented_lagrangian} .
\end{align}
The ADMM algorithm solves \eqref{augmented_lagrangian} by alternating the minimization of $\vec{x}$ and $\vec{v}$ with gradient ascent of $\vec{u}$, as specified in \algref{ADMM}.
ADMM is known to converge under convex $\ell(\cdot;\vec{y})$ and $\rho(\cdot)$, and other mild conditions (see \cite{Boyd:FTML:11}).

 \begin{algorithm}[t]
 \caption{ADMM \cite{Boyd:FTML:11}}
 \begin{algorithmic}[1] \label{alg:ADMM}
 \REQUIRE{$\ell(\cdot;\vec{y}),\rho(\cdot),
           \beta,\lambda,
           \vec{v}_0,\vec{u}_0$,
           and $K$} \label{line:ADMM_init}
 \FOR{$k = 1,2,\dots,K$}
 	\STATE{$\vec{x}_{k}=\arg\min_{\vec{x}}\{\ell(\vec{x};\vec{y})+\frac{\beta}{2}\norm{\vec{x}-\vec{v}_{k-1}+\vec{u}_{k-1}}^2\}$}\label{line:ADMM_x_update}
 	\STATE{$\vec{v}_k = \arg\min_{\vec{v}}\{\lambda\rho(\vec{v})+\frac{\beta}{2}\norm{\vec{v}-\vec{x}_k-\vec{u}_{k-1}}^2\}$}\label{line:ADMM_v_update}
 	\STATE{$\vec{u}_k = \vec{u}_{k-1}+\vec{x}_k-\vec{v}_k$}\label{line:ADMM_u_update}
 \ENDFOR
 \STATE{Return $\vec{x}_K$}
 \end{algorithmic}
 \end{algorithm}

\subsection{Plug-and-Play ADMM}

Importantly, \lineref{ADMM_v_update} of \algref{ADMM} can be recognized as variational \emph{denoising} of $\vec{x}_k + \vec{u}_{k-1}$ using 
regularization $\lambda\rho(\vec{x})$ and
quadratic loss $\ell(\vec{x};\vec{r})=\frac{1}{2\nu}\|\vec{x}-\vec{r}\|^2$,
where $\vec{r}=\vec{x}_k+\vec{u}_{k-1}$ at iteration $k$.
By ``denoising,'' we mean recovering $\vec{x}\true$ from noisy measurements $\vec{r}$ of the form
\begin{align}
\vec{r}=\vec{x}\true+\vec{e}, \quad \vec{e}\sim \mc{N}(\vec{0},\nu\vec{I})
\label{eq:rxe} ,
\end{align} 
for some variance $\nu>0$.

Image denoising has been studied for decades (see, e.g., the overviews \cite{Buades:MMS:05,Milanfar:SPM:13}), with the result that high performance methods are now readily available.
Today's state-of-the-art denoisers include those based on 
image-dependent filtering algorithms (e.g., BM3D \cite{Dabov:TIP:07}) or
deep neural networks (e.g., TNRD \cite{Chen:TPAMI:17}, DnCNN \cite{Zhang:TIP:17}).
Most of these denoisers are not variational in nature, i.e., they are not based on any explicit regularizer $\lambda\rho(\vec{x})$. 
 
Leveraging the denoising interpretation of ADMM, Venkatakrishnan, Bouman, and Wolhberg \cite{Venkatakrishnan:GSIP:13} proposed to replace \lineref{ADMM_v_update} of \algref{ADMM} with a call to a sophisticated image denoiser, such as BM3D, and dubbed their approach \emph{Plug-and-Play} (PnP) ADMM.
Numerical experiments show that PnP-ADMM works very well in most cases. 
However, when the denoiser used in PnP-ADMM comes with no explicit regularization $\rho(\vec{x})$, it is not clear what objective PnP-ADMM is minimizing, making PnP-ADMM convergence more difficult to characterize.
Similar PnP algorithms have been proposed using primal-dual methods \cite{Ono:SPL:17} and FISTA \cite{Kamilov:SPL:17} in place of ADMM.

Approximate message passing (AMP) algorithms \cite{Donoho:PNAS:09} also perform denoising at each iteration.
In fact, when $\vec{A}$ is large and i.i.d.\ Gaussian, AMP constructs an internal variable statistically equivalent to $\vec{r}$ in \eqref{rxe} \cite{Bayati:TIT:11}.
While the earliest instances of AMP assumed separable denoising (i.e., $[\vec{f}(\vec{x})]_n=f(x_n)~\forall n$ for some $f$) later instances, like \cite{Som:TSP:12,Donoho:TIT:13b}, considered non-separable denoising.
The paper \cite{Metzler:TIT:16} by Metzler, Maleki, and Baraniuk proposed to plug an image-specific denoising algorithm, like BM3D, into AMP.
Vector AMP, which extends AMP to the broader class of ``right rotationally invariant'' random matrices, was proposed in \cite{Rangan:ISIT:17}, and VAMP with image-specific denoising was proposed in \cite{Schniter:BASP:17}.
Rigorous analyses of AMP and VAMP under non-separable denoisers were performed in \cite{Berthier:17} and \cite{Fletcher:NIPS:18}, respectively.

\subsection{Regularization by Denoising (RED)} \label{sec:RED}

As discussed in the Introduction,
Romano, Elad, and Milanfar \cite{Romano:JIS:17} proposed a radically new way to exploit an image denoiser, which they call \emph{regularization by denoising} (RED).
Given an arbitrary image denoiser $\vec{f}:\Real^N\rightarrow\Real^N$, they proposed to construct an explicit regularizer of the form
\begin{align}
\rhoRED(\vec{x}) 
\defn \frac{1}{2}\vec{x}\tran(\vec{x}-\vec{f}(\vec{x}))
\label{eq:rhoRED}
\end{align}
to use within the variational framework \eqref{variational}. 
The advantage of using an explicit regularizer is that a wide variety of optimization algorithms can be used to solve \eqref{variational} and their convergence can be tractably analyzed.

In \cite{Romano:JIS:17}, numerical evidence is presented to show that image denoisers $\vec{f}(\cdot)$ are \emph{locally homogeneous} (LH), i.e., 
\begin{align}
(1+\epsilon)\vec{f}(\vec{x}) 
= \vec{f}\big((1+\epsilon)\vec{x}\big)~\forall \vec{x}
\label{eq:LH}
\end{align}
for sufficiently small $\epsilon\in\Real\setminus 0$.
For such denoisers, Romano et al.\ claim \cite[Eq.(28)]{Romano:JIS:17} that $\rhoRED(\cdot)$ obeys the gradient rule 
\begin{align}
\nabla \rhoRED(\vec{x}) 
&= \vec{x}-\vec{f}(\vec{x})
\label{eq:gradREDromano} .
\end{align}
If $\nabla \rhoRED(\vec{x})=\vec{x}-\vec{f}(\vec{x})$, then any minimizer $\hvec{x}$ of the variational objective under quadratic loss,
\begin{align}
\frac{1}{2\sigma^2}\|\vec{Ax}-\vec{y}\|^2 + \lambda \rhoRED(\vec{x})
&\defn \CRED(\vec{x})
\label{eq:CRED},
\end{align}
must yield $\nabla\CRED(\hvec{x})=\vec{0}$, i.e., must obey
\begin{align}
\vec{0} 
&= \frac{1}{\sigma^2}\vec{A}\tran(\vec{A}\hvec{x}-\vec{y}) + \lambda(\hvec{x}-\vec{f}(\hvec{x}))
\label{eq:fpRED} .
\end{align}
Based on this line of reasoning, Romano et al.\ proposed several iterative algorithms that find an $\hvec{x}$ satisfying the fixed-point condition \eqref{fpRED}, which we will refer to henceforth as ``RED algorithms.''

\section{Clarifications on RED} \label{sec:clarifications}

In this section, we first show that the gradient expression \eqref{gradREDromano} holds if and only if the denoiser $\vec{f}(\cdot)$ is LH and has Jacobian symmetry (JS).
We then establish that many popular denoisers lack JS, such as 
the median filter (MF) \cite{Huang:TASSP:79},
non-local means (NLM) \cite{Buades:CVPR:05},
BM3D \cite{Dabov:TIP:07}, 
TNRD \cite{Chen:TPAMI:17}, 
and
DnCNN \cite{Zhang:TIP:17}.
For such denoisers, the RED algorithms cannot be explained by $\rhoRED(\cdot)$ in \eqref{rhoRED}.
We also show a more general result: When a denoiser lacks JS, there exists no regularizer $\rho(\cdot)$ whose gradient expression matches \eqref{gradREDromano}.
Thus, the problem is not the specific form of $\rhoRED(\cdot)$ in \eqref{rhoRED} but rather the broader pursuit of explicit regularization.

\subsection{Preliminaries} \label{sec:prelim}

We first state some definitions and assumptions.
In the sequel, we denote 
the $i$th component of $\vec{f}(\vec{x})$ by $f_i(\vec{x})$,
the gradient of $f_i(\cdot)$ at $\vec{x}$ by
\begin{align}
\nabla f_i(\vec{x})
&\defn \begin{bmatrix}
\tfrac{\partial f_i(\vec{x})}{\partial x_1} 
& \cdots 
& \tfrac{\partial f_i(\vec{x})}{\partial x_N} 
\end{bmatrix}\tran 
\label{eq:gradient} ,
\end{align}
and the Jacobian of $\vec{f}(\cdot)$ at $\vec{x}$ by
\begin{align}
\Jf(\vec{x}) 
&\defn \mat{
\tfrac{\partial f_1(\vec{x})}{\partial x_1} 
& \tfrac{\partial f_1(\vec{x})}{\partial x_2}
& \dots 
& \tfrac{\partial f_1(\vec{x})}{\partial x_N} \\
\tfrac{\partial f_2(\vec{x})}{\partial x_1} 
& \tfrac{\partial f_2(\vec{x})}{\partial x_2} 
& \dots 
& \tfrac{\partial f_2(\vec{x})}{\partial x_N} \\
\vdots & \vdots & \ddots & \vdots\\
\tfrac{\partial f_N(\vec{x})}{\partial x_1} 
& \tfrac{\partial f_N(\vec{x})}{\partial x_2} 
& \dots 
& \tfrac{\partial f_N(\vec{x})}{\partial x_N} }
\label{eq:Jacobian} .
\end{align}

Without loss of generality, we take $[0,255]^N\subset \Real^N$ to be the set of possible images.
A given denoiser $\vec{f}(\cdot)$ may involve decision boundaries $\mc{D}\subset [0,255]^N$ at which its behavior changes suddenly.  
We assume that these boundaries are a closed set of measure zero and work instead with the open set $\mc{X}\defn (0,255)^N\setminus \mc{D}$, which contains almost all images.

We furthermore assume that $\vec{f}:\Real^N\rightarrow\Real^N$ is \emph{differentiable} on $\mc{X}$, which means \cite[p.212]{Rudin:Book:76} that, for any $\vec{x}\in\mc{X}$, there exists a matrix $\vec{J}\in\Real^{N\times N}$ for which 
\begin{align}
\lim_{\vec{w}\rightarrow \vec{0}} \frac{\|\vec{f}(\vec{x}+\vec{w})-\vec{f}(\vec{x}) - \vec{J}\vec{w}\|}{\|\vec{w}\|} = 0 
\label{eq:differentiable} .
\end{align}
When $\vec{J}$ exists, it can be shown \cite[p.216]{Rudin:Book:76} that $\vec{J}=\Jf(\vec{x})$.

\subsection{The RED Gradient} \label{sec:gradRED}

We first recall a result that was established in \cite{Romano:JIS:17}.
\begin{lemma}[Local homogeneity \cite{Romano:JIS:17}]\label{lem:LH}
Suppose that denoiser $\vec{f}(\cdot)$ is locally homogeneous. 
Then $[\Jf(\vec{x})]\vec{x} = \vec{f}(\vec{x})$.
\end{lemma}
\begin{proof}
Our proof is based on differentiability and avoids the need to define a directional derivative.
From \eqref{differentiable}, we have 
\begin{align}
0 
&= \lim_{\epsilon\rightarrow 0} \frac{\|\vec{f}(\vec{x}+\epsilon\vec{x})-\vec{f}(\vec{x}) - [\Jf(\vec{x})]\vec{x}\epsilon\|}{\|\epsilon\vec{x}\|} 
~\forall \vec{x}\in\mc{X}\\
&= \lim_{\epsilon\rightarrow 0} \frac{\|(1+\epsilon)\vec{f}(\vec{x})-\vec{f}(\vec{x}) - [\Jf(\vec{x})]\vec{x}\epsilon\|}{\|\epsilon\vec{x}\|} 
~\forall \vec{x}\in\mc{X}
\label{eq:LH2}\\
&= \lim_{\epsilon\rightarrow 0} \frac{\|\vec{f}(\vec{x}) - [\Jf(\vec{x})]\vec{x}\|}{\|\vec{x}\|} 
~\forall \vec{x}\in\mc{X}
\label{eq:LH3} ,
\end{align}
where \eqref{LH2} follows from local homogeneity \eqref{LH}.
\Eqref{LH3} implies that $[\Jf(\vec{x})]\vec{x} = \vec{f}(\vec{x})~\forall \vec{x}\in\mc{X}$.
\end{proof}

We now state one of the main results of this section.
\begin{lemma}[RED gradient] \label{lem:gradRED}
For $\rhoRED(\cdot)$ defined in \eqref{rhoRED},
\begin{align}
\nabla \rhoRED(\vec{x})
&= \vec{x} - \frac{1}{2} \vec{f}(\vec{x}) 
   - \frac{1}{2} [\Jf(\vec{x})]\tran\vec{x}
\label{eq:gradRED} .
\end{align}
\end{lemma}
\begin{proof}
For any $\vec{x}\in\mc{X}$ and $n=1,\dots,N$,
\begin{align}
\lefteqn{
\frac{\partial \rhoRED(\vec{x})}{\partial x_n}
= \frac{\partial}{\partial x_n} \frac{1}{2}
\sum_{i=1}^N \big( x_i^2 - x_i f_i(\vec{x}) \big) }\\
&= \frac{1}{2} \frac{\partial}{\partial x_n} \left(
x_n^2 
- x_n f_n(\vec{x}) 
+ \sum_{i\neq n} x_i^2 
- \sum_{i\neq n} x_i f_i(\vec{x}) \right)\\
&= \frac{1}{2} \left( 
2x_n 
- f_n(\vec{x}) 
- x_n 
\frac{\partial f_n(\vec{x})}{\partial x_n} 
- \sum_{i\neq n} x_i 
\frac{\partial f_i(\vec{x})}{\partial x_n} 
\right) \\
&= x_n - \frac{1}{2} f_n(\vec{x}) - \frac{1}{2} \sum_{i=1}^N x_i 
\frac{\partial f_i(\vec{x})}{\partial x_n}  
\label{eq:partialRED} \\
&= x_n - \frac{1}{2} f_n(\vec{x}) 
- \frac{1}{2} \left[ [\Jf(\vec{x})]\tran\vec{x} \right]_n ,
\end{align}
using the definition of $\Jf(\vec{x})$ from \eqref{Jacobian}.
Collecting $\{\frac{\partial \rhoRED(\vec{x})}{\partial x_n}\}_{n=1}^N$ into the gradient vector \eqref{gradREDromano} yields \eqref{gradRED}. 
\end{proof}

Note that the gradient expression \eqref{gradRED} differs from \eqref{gradREDromano}.

\begin{lemma}[Clarification on \eqref{gradREDromano}] \label{lem:gradREDromano}
Suppose that the denoiser $\vec{f}(\cdot)$ is locally homogeneous. 
Then the RED gradient expression \eqref{gradREDromano} 
holds if and only if $\Jf(\vec{x})=[\Jf(\vec{x})]\tran$.
\end{lemma}
\begin{proof}
If $\Jf(\vec{x})=[\Jf(\vec{x})]\tran$, then the last term in \eqref{gradRED} becomes $-\frac{1}{2} [\Jf(\vec{x})]\vec{x}$, which equals $-\frac{1}{2}\vec{f}(x)$ by \lemref{LH}, in which case \eqref{gradRED} agrees with \eqref{gradREDromano}.
But if $\Jf(\vec{x})\neq[\Jf(\vec{x})]\tran$, then \eqref{gradRED} differs from \eqref{gradREDromano}.
\end{proof}

\subsection{Impossibility of Explicit Regularization} \label{sec:impossibility}

For denoisers $\vec{f}(\cdot)$ that lack Jacobian symmetry (JS), \lemref{gradREDromano} establishes that the gradient expression \eqref{gradREDromano} does not hold. 
Yet \eqref{gradREDromano} leads to the fixed-point condition \eqref{fpRED} on which all RED algorithms in \cite{Romano:JIS:17} are based.
The fact that these algorithms work well in practice suggests that 
``$\nabla \rho(\vec{x}) = \vec{x}-\vec{f}(\vec{x})$''
is a desirable property for a regularizer $\rho(\vec{x})$ to have.
But the regularization $\rhoRED(\vec{x})$ in \eqref{rhoRED} does not lead to this property when $\vec{f}(\cdot)$ lacks JS.
Thus an important question is:
\begin{quote}\em 
Does there exist some other regularization $\rho(\cdot)$ for which $\nabla\rho(\vec{x})=\vec{x}-\vec{f}(\vec{x})$ when $\vec{f}(\cdot)$ is non-JS?
\end{quote}
The following theorem provides the answer.

\begin{theorem}[Impossibility] \label{thm:impossible}
Suppose that denoiser $\vec{f}(\cdot)$ has a non-symmetric Jacobian. 
Then there exists no regularization $\rho(\cdot)$ for which $\nabla\rho(\vec{x})=\vec{x}-\vec{f}(\vec{x})$.
\end{theorem}
\begin{proof}
To prove the theorem, we view $\vec{f}:\mc{X}\rightarrow\Real^N$ as a vector field.
Theorem~4.3.8 in \cite{Kantorovitz:Book:16} says that a vector field $\vec{f}$ is \emph{conservative} if and only if there exists a continuously differentiable potential $\bar{\rho}:\mc{X}\rightarrow\Real$ for which $\nabla\bar{\rho}=\vec{f}$.
Furthermore, Theorem~4.3.10 in \cite{Kantorovitz:Book:16} says that if $\vec{f}$ is conservative, then the Jacobian $J\vec{f}$ is symmetric.
Thus, by the contrapositive, if the Jacobian $J\vec{f}$ is \emph{not} symmetric, then no such potential $\bar{\rho}$ exists.

To apply this result to our problem, we define
\begin{align}
\rho(\vec{x}) \defn \frac{1}{2} \|\vec{x}\|^2 - \bar{\rho}(\vec{x}) 
\end{align}
and notice that
\begin{align}
\nabla\rho(\vec{x}) = \vec{x} - \nabla\bar{\rho}(\vec{x})
= \vec{x} - \vec{f}(\vec{x}) 
\label{eq:rhobar}.
\end{align}
Thus, if $J\vec{f}(\vec{x})$ is non-symmetric, then $J[\vec{x}-\vec{f}(\vec{x})]=\vec{I}-J\vec{f}(\vec{x})$ is non-symmetric, which means that there exists no $\rho$ for which \eqref{rhobar} holds.
\end{proof}
Thus, the problem is not the specific form of $\rhoRED(\cdot)$ in \eqref{rhoRED} but rather the broader pursuit of explicit regularization.
We note that the notion of conservative vector fields was discussed in \cite[App. A]{Sreehari:TCI:16} in the context of PnP algorithms, whereas here we discuss it in the context of RED.

\subsection{Analysis of Jacobian Symmetry} \label{sec:JS}

The previous sections motivate an important question: 
Do commonly-used image denoisers have sufficient JS?

For some denoisers, JS can be studied analytically.
For example, consider the ``transform domain thresholding'' (TDT) denoisers of the form
\begin{align}
\vec{f}(\vec{x}) 
\defn 
\vec{W}\tran\vec{g}(\vec{Wx})
\label{eq:fTD} ,
\end{align}
where $\vec{g}(\cdot)$ performs componentwise (e.g., soft or hard) thresholding
and $\vec{W}$ is some transform, as occurs in the context of 
wavelet shrinkage \cite{Donoho:Bio:94},
with or without cycle-spinning \cite{Coifman:Chap:95}.
Using $g_n'(\cdot)$ to denote the derivative of $g_n(\cdot)$, we have
\begin{align}
\frac{\partial f_n(\vec{x})}{\partial x_q}
&= \sum_{i=1}^N w_{in} g_i'\Bigg(\sum_{j=1}^N w_{ij} x_j\Bigg) w_{iq}
= \frac{\partial f_q(\vec{x})}{\partial x_n} ,
\end{align}
and so the Jacobian of $\vec{f}(\cdot)$ is perfectly symmetric.

Another class of denoisers with perfectly symmetric Jacobians are those that produce MAP or MMSE optimal $\hvec{x}$ under some assumed prior $\pxhat$. 
In the MAP case, $\hvec{x}$ minimizes (over $\vec{x}$) the cost $c(\vec{x};\vec{r})=\frac{1}{2\nu}\|\vec{x}-\vec{r}\|^2-\ln\pxhat(\vec{x})$ for noisy input $\vec{r}$.
If we define $\phi(\vec{r})\defn\min_{\vec{x}} c(\vec{x};\vec{r})$, known as the Moreau-Yosida envelope of $-\ln\pxhat$, then $\hvec{x}=\vec{f}(\vec{r})=\vec{r}-\nu\nabla\phi(\vec{r})$, as discussed in \cite{Parikh:FTO:13}
(See also \cite{Ong:18} for insightful discussions in the context of image denoising.)
The elements in the Jacobian are therefore
$[J\vec{f}(\vec{r})]_{n,q}
 = \frac{\partial f_n(\vec{r})}{\partial r_q}
 = \delta_{n-q} - \nu\frac{\partial^2 \phi(\vec{r})}{\partial r_q\partial r_n}$,
and so the Jacobian matrix is symmetric.
In the MMSE case, we have that 
$\vec{f}(\vec{r})=\vec{r}-\nabla\rhoTR(\vec{r})$ 
for $\rhoTR(\cdot)$ defined in \eqref{rhoTR} (see \lemref{gradTR}), and so 
$[J\vec{f}(\vec{r})]_{n,q} = \delta_{n-q} - \frac{\partial^2 \rhoTR(\vec{r})}{\partial r_q\partial r_n}$, again implying that the Jacobian is symmetric.
But it is difficult to say anything about the Jacobian symmetry of \emph{approximate} MAP or MMSE denoisers.

Now let us consider the more general class of denoisers
\begin{align}
\vec{f}(x)
&= \vec{W}(\vec{x})\vec{x}
\label{eq:fPL} ,
\end{align}
sometimes called ``pseudo-linear'' \cite{Milanfar:SPM:13}. 
For simplicity, we assume that $\vec{W}(\cdot)$ is differentiable on $\mc{X}$.
In this case, using the chain rule, we have
\begin{align}
\frac{\partial f_n(\vec{x})}{\partial x_q}
&= w_{nq}(\vec{x}) 
   + \sum_{i=1}^N \frac{\partial w_{ni}(\vec{x})}{\partial x_q} x_i ,
\end{align}
and so the following are sufficient conditions for Jacobian symmetry.
\begin{enumerate}
\item 
$\vec{W}(\vec{x})$ is symmetric $\forall \vec{x}\in\mc{X}$,
\item
$\sum_{i=1}^N \frac{\partial w_{ni}(\vec{x})}{\partial x_q} x_i 
=\sum_{i=1}^N \frac{\partial w_{qi}(\vec{x})}{\partial x_n} x_i
~\forall \vec{x}\in\mc{X}$.
\end{enumerate}
When $\vec{W}$ is $\vec{x}$-invariant (i.e., $\vec{f}(\cdot)$ is linear) and symmetric, both of these conditions are satisfied.  
This latter case was exploited for RED in \cite{Teodoro:18}.
The case of non-linear $\vec{W}(\cdot)$ is more complicated.
Although $\vec{W}(\cdot)$ can be symmetrized (see \cite{Milanfar:JIS:13,Milanfar:ICIP:16}),
it is not clear whether the second condition above will be satisfied.

\subsection{Jacobian Symmetry Experiments} \label{sec:gradREDnum}

For denoisers that do not admit a tractable analysis, we can still evaluate the Jacobian of $\vec{f}(\cdot)$ at $\vec{x}$ numerically via
\begin{align}
\frac{f_i(\vec{x}+\epsilon\vec{e}_n)-f_i(\vec{x}-\epsilon\vec{e}_n)}{2\epsilon}  
\defn \big[ \widehat{\Jf}(\vec{x}) \big]_{i,n} ,
\end{align}
where $\vec{e}_n$ denotes the $n$th column of $\vec{I}_N$ and $\epsilon>0$ is small ($\epsilon=1\times10^{-3}$ in our experiments).  
For the purpose of quantifying JS, we define the normalized error metric
\begin{align}
e^J_{\vec{f}}(\vec{x}) 
\defn \frac{\big\|\widehat{\Jf}(\vec{x}) - [\widehat{\Jf}(\vec{x})]\tran\big\|_F^2}
{\|\widehat{\Jf}(\vec{x})\|_F^2} ,
\end{align}
which should be nearly zero for a symmetric Jacobian.

\tabref{Jacobian} shows\footnote{Matlab code for the experiments is available at \url{http://www2.ece.ohio-state.edu/~schniter/RED/index.html}.} the average value of $e^J_{\vec{f}}(\vec{x})$ 
for $17$ different image patches\footnote{\label{foot:images}%
We used the center $16\times 16$ patches of the standard
Barbara, Bike, Boats, Butterfly, Cameraman, Flower, Girl, Hat, House, Leaves, Lena, Parrots, Parthenon, Peppers, Plants, Raccoon, and Starfish test images.}
of size $16\times 16$, using denoisers that assumed a noise variance of $25^2$.
The denoisers tested were 
the TDT from \eqref{fTD} with the 2D Haar wavelet transform and soft-thresholding,
the median filter (MF) \cite{Huang:TASSP:79} with a $3\times 3$ window,
non-local means (NLM) \cite{Buades:CVPR:05}, 
BM3D \cite{Dabov:TIP:07},
TNRD \cite{Chen:TPAMI:17},
and
DnCNN \cite{Zhang:TIP:17}.
\tabref{Jacobian} shows that the Jacobians of all but the TDT denoiser are far from symmetric. 

\putTable{Jacobian}
{Average Jacobian-symmetry error on 16$\times$16 images}
{
\begin{tabular}{|c||c|c|c|c|c|c|} \hline
& TDT & MF & NLM & BM3D & TNRD & DnCNN\\ \hline\hline
$e^J_{\vec{f}}(\vec{x})$ 
& 5.36e-21 & 1.50 & 0.250 & 1.22 & 0.0378 & 0.0172\\ \hline 
\end{tabular}%
}

Jacobian symmetry is of secondary interest;
what we really care about is the accuracy of the RED gradient expressions 
\eqref{gradREDromano} and \eqref{gradRED}.
To assess gradient accuracy, we numerically evaluated the gradient of $\rhoRED(\cdot)$ at $\vec{x}$ using
\begin{align}
\frac{\rhoRED(\vec{x}+\epsilon\vec{e}_n)-\rhoRED(\vec{x}-\epsilon\vec{e}_n)}{2\epsilon}  
\defn \big[ \widehat{ \nabla \rhoRED}(\vec{x}) \big]_{n} 
\end{align}
and compared the result to the analytical expressions \eqref{gradREDromano} and \eqref{gradRED}.
\tabref{gradRED} reports the normalized gradient error 
\begin{align}
e^\nabla_{\vec{f}}(\vec{x})
\defn
\frac{\|\nabla\rhoRED(\vec{x})-\widehat{\nabla\rhoRED}(\vec{x})\|^2}{\|\widehat{\nabla\rhoRED}(\vec{x})\|^2}
\end{align}
for the same $\epsilon$, images, and denoisers used in \tabref{Jacobian}.
The results in \tabref{gradRED} show that, for all tested denoisers, the numerical gradient $\widehat{\nabla\rhoRED}(\cdot)$ closely matches the analytical expression for $\nabla\rhoRED(\cdot)$ from \eqref{gradRED}, but not that from \eqref{gradREDromano}.
The mismatch between $\widehat{\nabla\rhoRED}(\cdot)$ and $\nabla\rhoRED(\cdot)$ from \eqref{gradREDromano} is partly due to insufficient JS and partly due to insufficient LH, as we establish below.

\putTable{gradRED}
{Average gradient error on 16$\times$16 images}
{
\resizebox{\columnwidth}{!}{%
\begin{tabular}{|c||c|c|c|c|c|c|}
\hline
$e^\nabla_{\vec{f}}(\vec{x})$ 
& TDT & MF & NLM & BM3D & TNRD & DnCNN\\ \hline \hline
$\nabla\rhoRED(\vec{x})$ from \eqref{gradREDromano} 
& 0.381 & 0.904 & 0.829 & 0.790 & 0.416 & 1.76 \\ \hline 
$\nabla\rhoRED(\vec{x})$ from \eqref{gradREDint}
& 0.381 & 1.78e-21 & 0.0446 & 0.447 & 0.356 & 1.69 \\ \hline 
$\nabla\rhoRED(\vec{x})$ from \eqref{gradRED}
& 4.68e-19 & 1.75e-21 & 1.32e-20 & 4.80e-14 & 3.77e-19 & 6.76e-13 \\ \hline
\end{tabular}%
}
}

\subsection{Local Homogeneity Experiments} \label{sec:LH}

Recall that the TDT denoiser has a symmetric Jacobian, both theoretically and empirically.
Yet \tabref{gradRED} reports a disagreement between the $\nabla\rhoRED(\cdot)$ expressions \eqref{gradREDromano} and \eqref{gradRED} for TDT.
We now show that this disagreement is due to insufficient local homogeneity (LH).

To do this, we introduce yet another RED gradient expression,
\begin{align}
\nabla\rhoRED(\vec{x})
&\eqLH \vec{x} - \frac{1}{2}[J\vec{f}(\vec{x})]\vec{x} -  \frac{1}{2}[J\vec{f}(\vec{x})]\tran\vec{x} 
\label{eq:gradREDint} ,
\end{align}
which results from combining \eqref{gradRED} with \lemref{LH}.
Here, $\eqLH$ indicates that \eqref{gradREDint} holds under LH.
In contrast, 
the gradient expression \eqref{gradREDromano} holds under \emph{both} LH and Jacobian symmetry,
while the gradient expression \eqref{gradRED} holds in general (i.e., even in the absence of LH and/or Jacobian symmetry). 
We also introduce two normalized error metrics for LH, 
\begin{align}
\eLHb(\vec{x})
&\defn \frac{\big\|\vec{f}((1+\epsilon)\vec{x}) - (1+\epsilon)\vec{f}(\vec{x})\big\|^2}{\|(1+\epsilon)\vec{f}(\vec{x})\|^2} 
\label{eq:eLHb}\\
\eLHa(\vec{x})
&\defn \frac{\big\|[\widehat{J\vec{f}}(\vec{x})]\vec{x}-\vec{f}(\vec{x})\big\|^2} {\|\vec{f}(\vec{x})\|^2} 
\label{eq:eLHa}.
\end{align}
which should both be nearly zero for LH $\vec{f}(\cdot)$.
Note that $\eLHb$ quantifies LH according to definition \eqref{LH} and closely matches the numerical analysis of LH in \cite{Romano:JIS:17}.
Meanwhile, $\eLHa$ quantifies LH according to \lemref{LH} and to how LH is actually used in the gradient expressions \eqref{gradREDromano} and \eqref{gradREDint}.

\putTable{LH}
{Average local-homogeneity error on 16$\times$16 images}
{
\resizebox{\columnwidth}{!}{%
\begin{tabular}{|c||c|c|c|c|c|c|}
\hline
& TDT & MF & NLM & BM3D & TNRD & DnCNN\\ \hline \hline
$\eLHb(\vec{x})$ 
& 2.05e-8 & 0 & 1.41e-8 & 7.37e-7 & 2.18e-8 & 1.63e-8 \\ \hline
$\eLHa(\vec{x})$ 
& 0.0205 & 2.26e-23 & 0.0141 & 3.80e4 & 2.18e-2 & 0.0179 \\ \hline 
\end{tabular}%
}
}
The middle row of \tabref{gradRED} reports the average gradient error of the gradient expression \eqref{gradREDint}, and
\tabref{LH} reports average LH error for the metrics $\eLHb$ and $\eLHa$.
There we see that the average $\eLHb$ error is small for all denoisers, consistent with the experiments in \cite{Romano:JIS:17}.
But the average $\eLHa$ error is several orders of magnitude larger (for all but the MF denoiser).
We also note that the value of $\eLHa$ for BM3D is several orders of magnitude higher than for the other denoisers.
This result is consistent with \figref{cost_figs}, which shows that the cost function associated with BM3D is much less smooth than that of the other denoisers.
As discussed below, these seemingly small imperfections in LH have a significant effect on the RED gradient expressions \eqref{gradREDromano} and \eqref{gradREDint}.

Starting with the TDT denoiser, \tabref{gradRED} shows that the gradient error on \eqref{gradREDint} is large, which can only be caused by insufficient LH.
The insufficient LH is confirmed in \tabref{LH}, which shows that the value of $\eLHa(\vec{x})$ for TDT is non-negligible, especially in comparison to the value for MF.

Continuing with the MF denoiser, \tabref{Jacobian} indicates that its Jacobian is far from symmetric,
while \tabref{LH} indicates that it is LH.
The gradient results in \tabref{gradRED} are consistent with these behaviors: the $\nabla\rhoRED(\vec{x})$ expression \eqref{gradREDint} is accurate on account of LH being satisfied, but the $\nabla\rhoRED(\vec{x})$ expression \eqref{gradREDromano} is inaccurate on account of a lack of JS.

The results for the remaining denoisers NLM, BM3D, TNRD, and BM3D show a common trend: they have non-trivial levels of \emph{both} JS error (see \tabref{Jacobian}) and LH error (see \tabref{LH}).
As a result, the gradient expressions \eqref{gradREDromano} and \eqref{gradREDint} are \emph{both} inaccurate (see \tabref{gradRED}).

In conclusion, the experiments in this section show that the RED gradient expressions \eqref{gradREDromano} and \eqref{gradREDint} are very sensitive to small imperfections in LH.
Although the experiments in \cite{Romano:JIS:17} suggested that many popular image denoisers are approximately LH, our experiments suggest that their levels of LH are insufficient to maintain the accuracy of the RED gradient expressions \eqref{gradREDromano} and \eqref{gradREDint}.

\subsection{Hessian and Convexity}

From \eqref{partialRED}, the $(n,j)$th element of the Hessian of $\rhoRED(\vec{x})$ equals
\begin{align}
\lefteqn{ 
\frac{\partial^2 \rhoRED(\vec{x})}{\partial x_n \partial x_j}
= \frac{\partial}{\partial x_j}\left(
x_n - \frac{1}{2} f_n(\vec{x}) - \frac{1}{2} \sum_{i=1}^N x_i 
\frac{\partial f_i(\vec{x})}{\partial x_n}  
\right) }\\
&= \delta_{n-j} - \frac{1}{2} \frac{\partial f_n(\vec{x})}{\partial x_j} 
- \frac{1}{2}  \frac{\partial f_j(\vec{x})}{\partial x_n}  
- \frac{1}{2}  x_j \frac{\partial^2 f_j(\vec{x})}{\partial x_n\partial x_j} 
\qquad\quad \nonumber\\&\quad
- \frac{1}{2} \sum_{i\neq j} x_i \frac{\partial^2 f_i(\vec{x})}{\partial x_n \partial x_j} \\
&= \delta_{n-j} - \frac{1}{2} \frac{\partial f_n(\vec{x})}{\partial x_j} 
- \frac{1}{2}  \frac{\partial f_j(\vec{x})}{\partial x_n}  
- \frac{1}{2} \sum_{i=1}^N x_i \frac{\partial^2 f_i(\vec{x})}{\partial x_n \partial x_j} .
\end{align}
where $\delta_{k}=1$ if $k=0$ and otherwise $\delta_{k}=0$.
Thus, the Hessian of $\rhoRED(\cdot)$ at $\vec{x}$ equals
\begin{align}
H\rhoRED(\vec{x})
&= \vec{I} - \frac{1}{2} \Jf(\vec{x}) - \frac{1}{2} [\Jf(\vec{x})]\tran
- \frac{1}{2} \sum_{i=1}^N x_i Hf_i(\vec{x}) 
\label{eq:hessRED} .
\end{align}
This expression can be contrasted with the Hessian expression from \cite[(60)]{Romano:JIS:17}, which reads
\begin{align}
\vec{I} - \Jf(\vec{x}) 
\label{eq:hessREDromano} .
\end{align}

Interestingly, \eqref{hessRED} differs from \eqref{hessREDromano} even when the denoiser has a symmetric Jacobian $\Jf(\vec{x})$.
One implication is that, even if eigenvalues of $\Jf(\vec{x})$ are limited to the interval $[0,1]$, the Hessian $H\rhoRED(\vec{x})$ may not be positive semi-definite due to the last term in \eqref{hessRED}, with possibly negative implications on the convexity of $\rhoRED(\cdot)$.
That said, the RED algorithms do not actually minimize the variational objective $\ell(\vec{x};\vec{y})+\lambda\rhoRED(\vec{x})$ for common denoisers $\vec{f}(\cdot)$ (as established in \secref{trajectory}), and so the convexity of $\rhoRED(\cdot)$ may not be important in practice.
We investigate the convexity of $\rhoRED(\cdot)$ numerically in \secref{cost}.

\subsection{Example RED-SD Trajectory} \label{sec:trajectory}

We now provide an example of how the RED algorithms from \cite{Romano:JIS:17} do not necessarily minimize the variational objective $\ell(\vec{x};\vec{y})+\lambda\rhoRED(\vec{x})$.

For a trajectory $\{\vec{x}_k\}_{k=1}^K$ produced by the steepest-descent (SD) RED algorithm from \cite{Romano:JIS:17},
\figref{poor_behaved} plots, versus iteration $k$, the RED Cost $\CRED(\vec{x}_k)$ from \eqref{CRED} and the error on the fixed-point condition \eqref{fpRED}, i.e., $\|\vec{g}(\vec{x}_k)\|^2$ with 
\begin{align}
\vec{g}(\vec{x}) 
&\defn \frac{1}{\sigma^2}\vec{A}\tran(\vec{Ax}-\vec{y})
+\lambda\big(\vec{x}-\vec{f}(\vec{x})\big)
\label{eq:gradCREDalg} .
\end{align}
For this experiment, we used 
the $3\times 3$ median-filter for $\vec{f}(\cdot)$,
the Starfish image,
and
noisy measurements $\vec{y}=\vec{x}+\mc{N}(\vec{0},\sigma^2\vec{I})$ 
with $\sigma^2=20$ (i.e., $\vec{A}=\vec{I}$ in \eqref{CRED}).

\putFrag{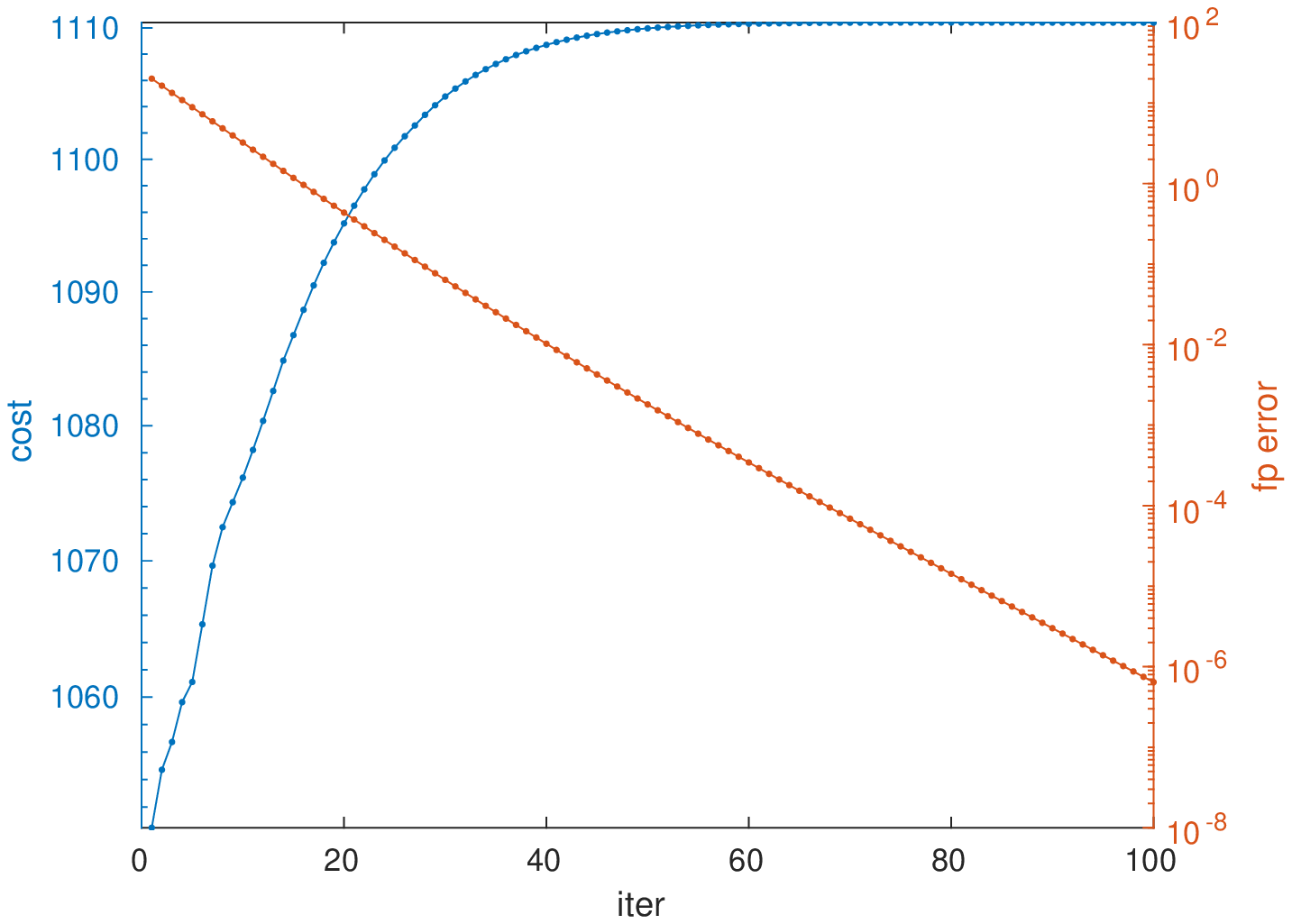}
	{RED cost $\CRED(\vec{x}_k)$ and
         fixed-point error $\|\vec{A}\tran(\vec{Ax}_k-\vec{y})/\sigma^2
         +\lambda(\vec{x}_k-\vec{f}(\vec{x}_k))\|^2$ versus iteration $k$
         for $\{\vec{x}_k\}_{k=1}^K$ 
         produced by the RED-SD algorithm from \cite{Romano:JIS:17}.
         Although the fixed-point condition is asymptotically satisfied, 
         the RED cost does not decrease with $k$.}
	{\figsize}
  	{\psfrag{cost}[b][b][0.8]{\sf $\CRED(\vec{x}_k)$}
	 \psfrag{fp error}[t][t][0.8]{\sf 
                  $\big\|\vec{A}\tran(\vec{Ax}_k-\vec{y})/\sigma^2
                  +\lambda\big(\vec{x}_k-\vec{f}(\vec{x}_k)\big)\big\|^2$}
	 \psfrag{iter}[t][t][0.7]{\sf iteration $k$} }
        {trim=10pt 10pt 0pt 20pt}
 
\Figref{poor_behaved} shows that, although the RED-SD algorithm asymptotically  satisfies the fixed-point condition \eqref{fpRED}, the RED cost function $\CRED(\vec{x}_k)$ does not decrease with $k$, as would be expected if the RED algorithms truly minimized the RED cost $\CRED(\cdot)$.
This behavior implies that any optimization algorithm that monitors the objective value $\CRED(\vec{x}_k)$ for, say, backtracking line-search (e.g., the FASTA algorithm \cite{Goldstein:14}), is difficult to apply in the context of RED.

\subsection{Visualization of RED Cost and RED-Algorithm Gradient} \label{sec:cost}
We now show visualizations of the RED cost $\CRED(\vec{x})$ from \eqref{CRED} and the RED algorithm's gradient field $\vec{g}(\vec{x})$ from \eqref{gradCREDalg}, for various image denoisers.
For this experiment, we used 
the Starfish image,
noisy measurements $\vec{y}=\vec{x}+\mc{N}(\vec{0},\sigma^2\vec{I})$ 
with $\sigma^2=100$ (i.e., $\vec{A}=\vec{I}$ in \eqref{CRED} and \eqref{gradCREDalg}),
and $\lambda$ optimized over a grid (of $20$ values logarithmically spaced between $0.0001$ and $1$) for each denoiser, so that the PSNR of the RED fixed-point $\hvec{x}$ is maximized.

\Figref{cost_figs} plots the RED cost $\CRED(\vec{x})$ and the RED algorithm's gradient field $\vec{g}(\vec{x})$ for the TDT, MF, NLM, BM3D, TNRD, and DnCNN denoisers.
To visualize these quantities in two dimensions, we plotted values of $\vec{x}$ centered at the RED fixed-point $\hvec{x}$ and varying along two randomly chosen directions.
The figure shows that the minimizer of $\CRED(\vec{x})$ does not coincide with the fixed-point $\hvec{x}$, and that the RED cost $\CRED(\cdot)$ is not always smooth or convex.

\begin{figure}
\newcommand{\hi}{1.35in}
\newcommand{\szz}{0.75}
\begin{tabular}{@{}c@{}c@{}}
\psfrag{TDT}[b][b][\szz]{\sf TDT}
\psfrag{alpha}[Bl][Bl][\szz]{$\alpha$}\psfrag{beta}[t][t][\szz]{$\beta$}
\includegraphics[height=\hi,clip]{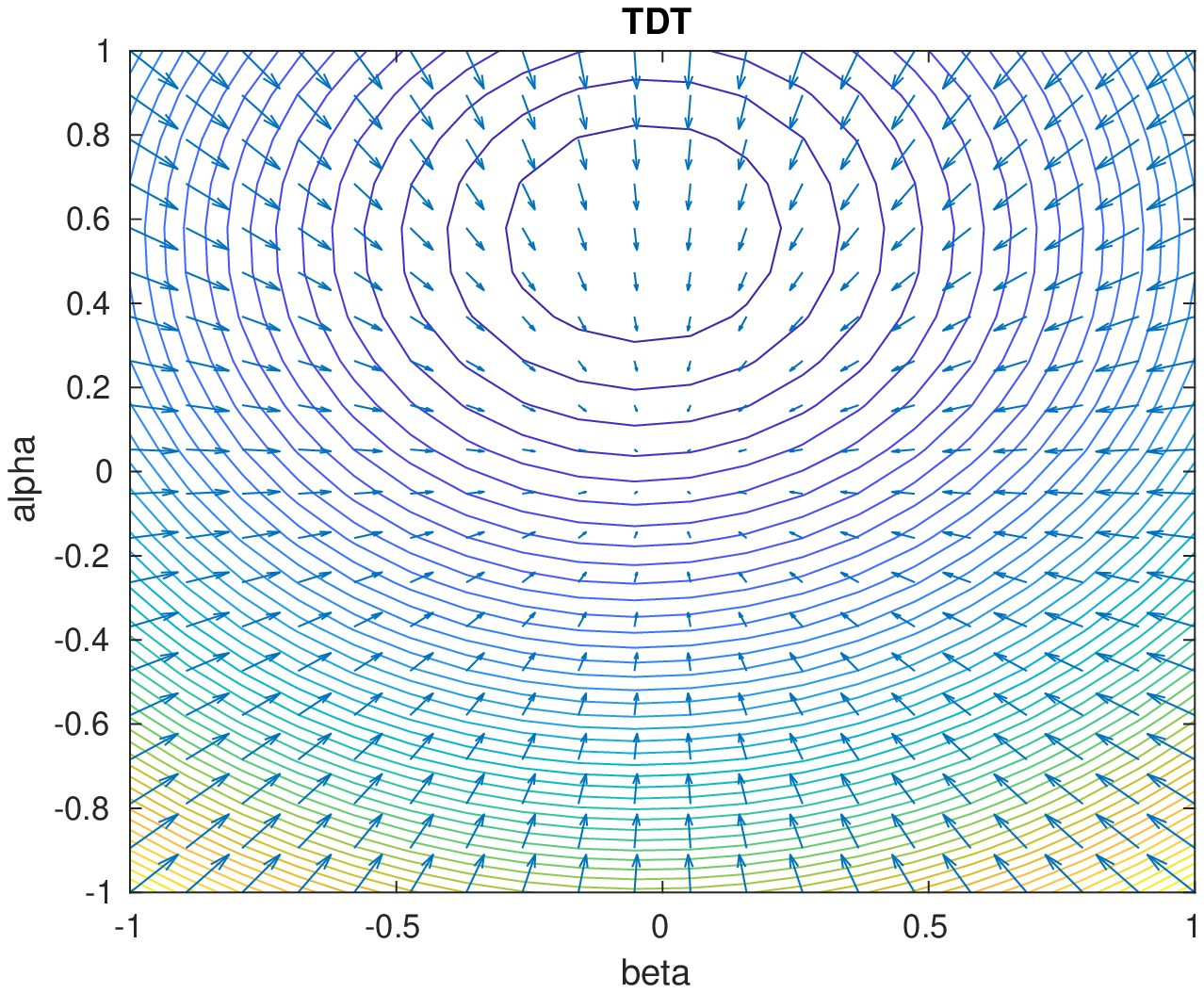}&
\psfrag{MF}[b][b][\szz]{\sf MF}
\psfrag{alpha}[Bl][Bl][\szz]{$\alpha$}\psfrag{beta}[t][t][\szz]{$\beta$}
\includegraphics[height=\hi,clip]{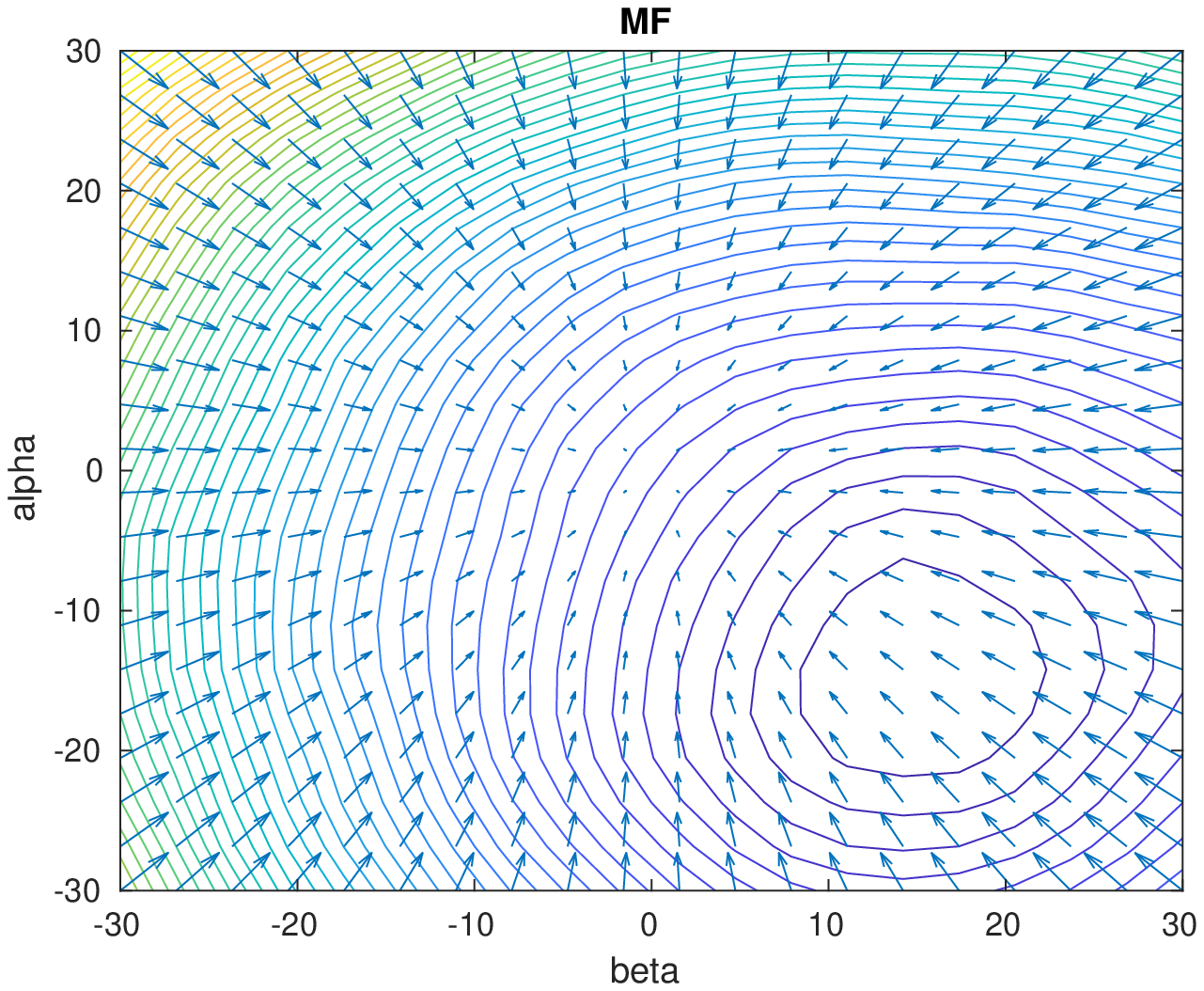}\\[2mm]
\psfrag{NLM}[b][b][\szz]{\sf NLM}
\psfrag{alpha}[Bl][Bl][\szz]{$\alpha$}\psfrag{beta}[t][t][\szz]{$\beta$}
\includegraphics[height=\hi,clip]{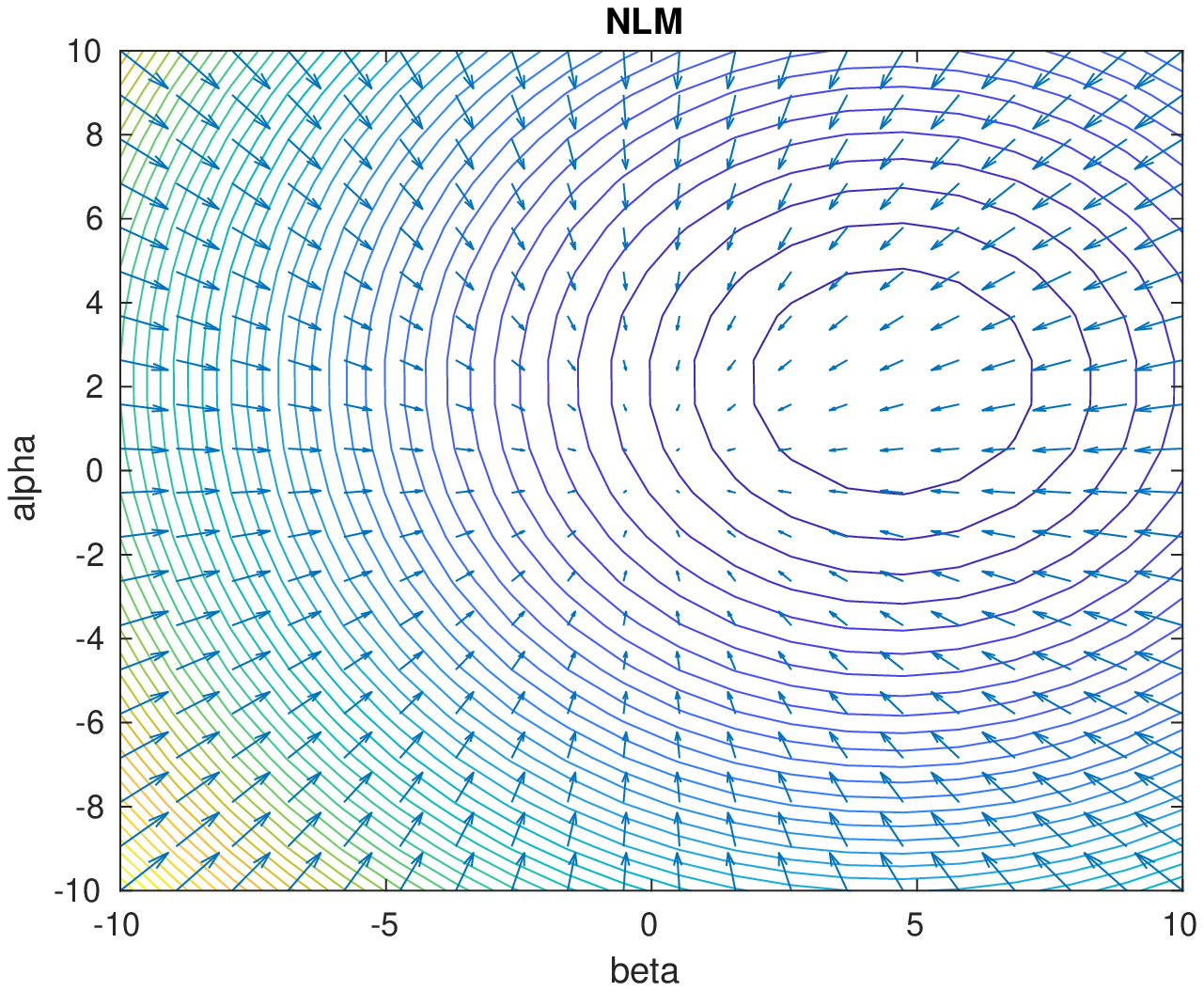}&
\psfrag{BM3D}[b][b][\szz]{\sf BM3D}
\psfrag{alpha}[Bl][Bl][\szz]{$\alpha$}\psfrag{beta}[t][t][\szz]{$\beta$}
\includegraphics[height=\hi,trim=-15 -15 -15 -15,clip]{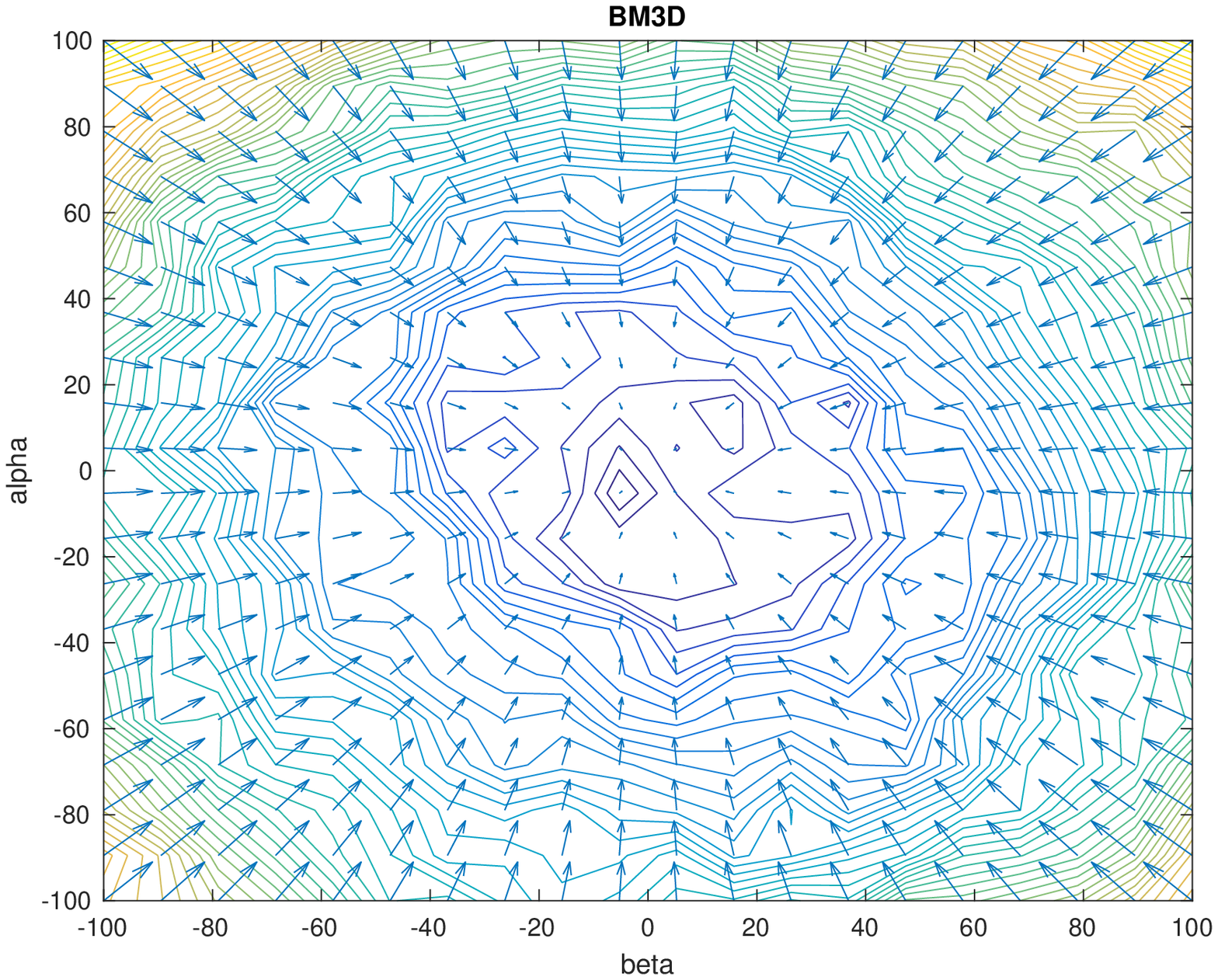}\\[2mm]
\psfrag{TRND}[b][b][\szz]{\sf TNRD}
\psfrag{alpha}[Bl][Bl][\szz]{$\alpha$}\psfrag{beta}[t][t][\szz]{$\beta$}
\includegraphics[height=\hi,clip]{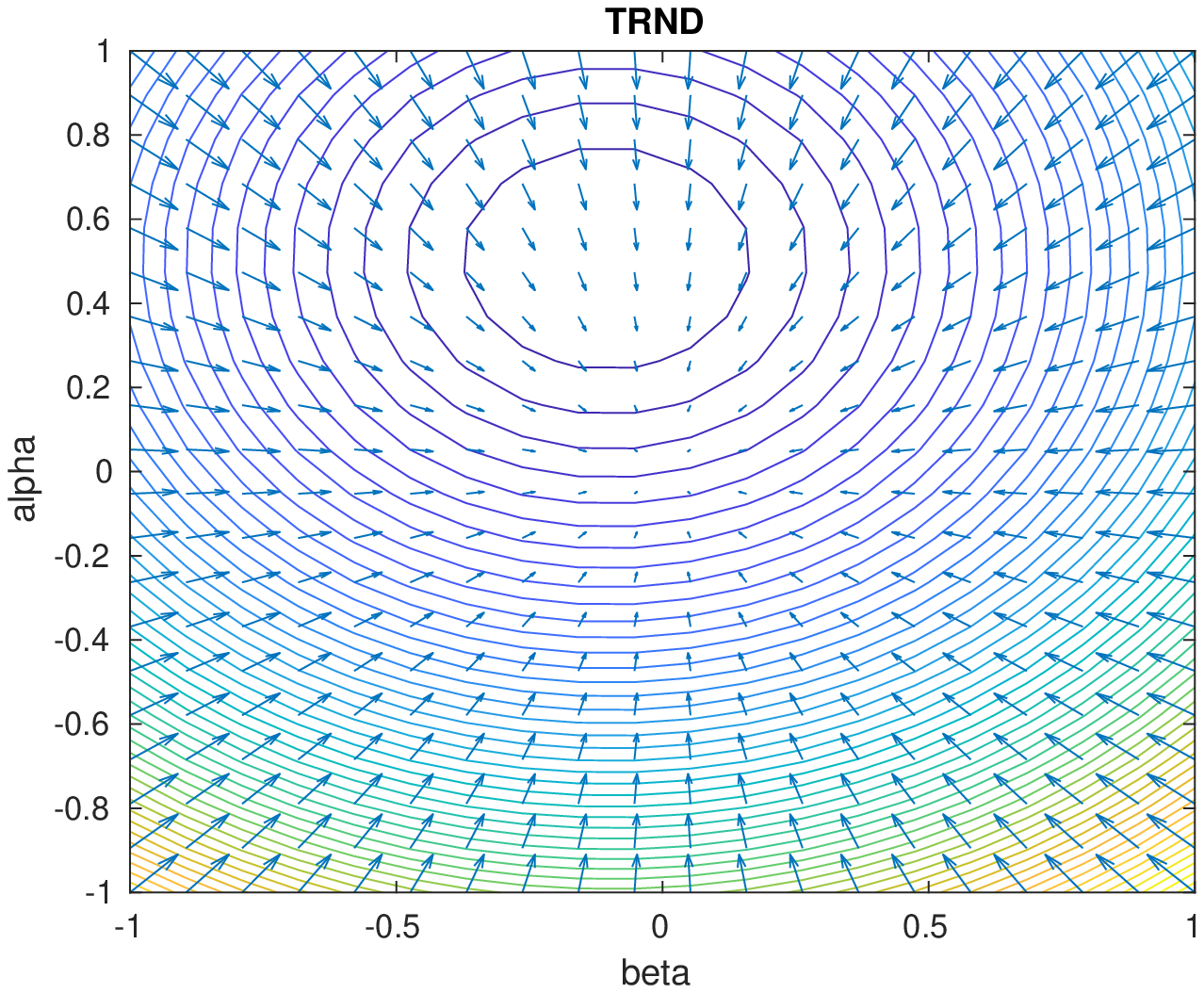}&
\psfrag{DnCNN}[b][b][\szz]{\sf DnCNN}
\psfrag{alpha}[Bl][Bl][\szz]{$\alpha$}\psfrag{beta}[t][t][\szz]{$\beta$}
\includegraphics[height=\hi,clip]{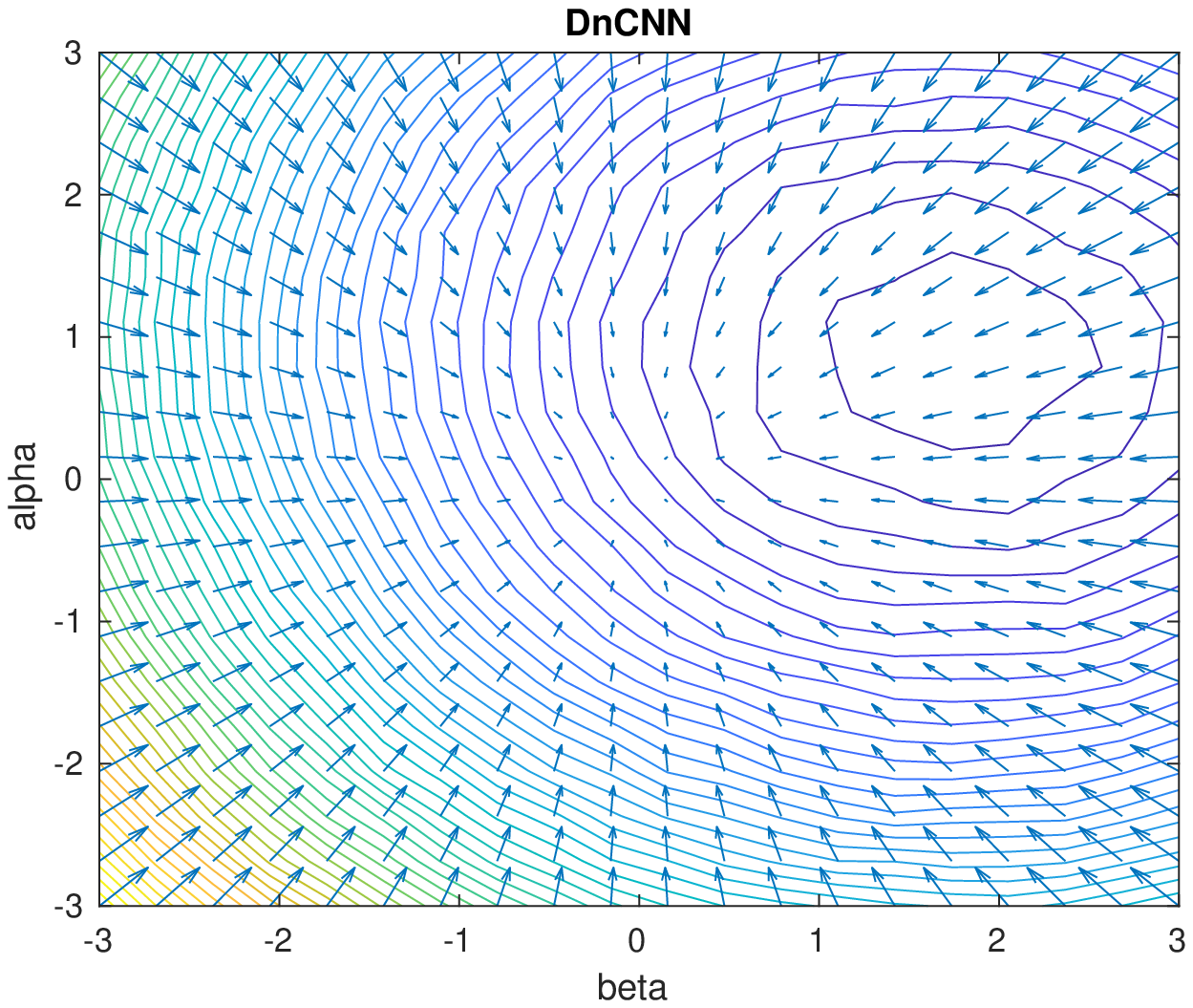}\\
\end{tabular}
\caption{Contours show RED cost $\CRED(\vec{x}_{\alpha,\beta})$ from \eqref{CRED} and arrows show RED-algorithm gradient field $\vec{g}(\vec{x}_{\alpha,\beta})$ from \eqref{gradCREDalg} versus $(\alpha,\beta)$,
where $\vec{x}_{\alpha,\beta}=\hvec{x}+\alpha\vec{e}_1+\beta\vec{e}_2$ 
with randomly chosen $\vec{e}_1$ and $\vec{e}_2$.  
The subplots show that the minimizer of $\CRED(\vec{x}_{\alpha,\beta})$ is not the fixed-point $\hvec{x}$, and that $\CRED(\cdot)$ may be non-smooth and/or non-convex.}
\label{fig:cost_figs}
\end{figure}

\section{Score-Matching by Denoising} \label{sec:new}

As discussed in \secref{RED}, the RED algorithms proposed in \cite{Romano:JIS:17} are explicitly based on gradient rule 
\begin{align}
\nabla\rho(\vec{x}) = \vec{x}-\vec{f}(\vec{x})
\label{eq:desired} .
\end{align}
This rule appears to be useful, since these algorithms work very well in practice.
But \secref{clarifications} established that $\rhoRED(\cdot)$ from \eqref{rhoRED} does not usually satisfy \eqref{desired}.
We are thus motived to seek an alternative explanation for the RED algorithms.
In this section, we explain them through a framework that we call \emph{score-matching by denoising} (SMD).

\subsection{Tweedie Regularization} \label{sec:rhoTR} 
As a precursor to the SMD framework, we first propose a technique based on what we will call \emph{Tweedie regularization}.

Recall the measurement model \eqref{rxe} used to define the ``denoising'' problem, repeated in \eqref{rxe1} for convenience:
\begin{align}
\vec{r} = \vec{x}\true+\vec{e}, \quad \vec{e}\sim \mc{N}(\vec{0},\nu\vec{I})
\label{eq:rxe1} .
\end{align} 
To avoid confusion, we will refer to $\vec{r}$ as ``pseudo-measurements'' and $\vec{y}$ as ``measurements.'' 
From \eqref{rxe1}, the likelihood of $\vec{x}\true$ is $p(\vec{r}|\vec{x}\true;\nu) = \mc{N}(\vec{r};\vec{x}\true,\nu\vec{I})$.

Now, suppose that we model the true image $\vec{x}\true$ as a realization of a random vector $\vec{x}$ with prior pdf $\pxhat$.
We write ``$\pxhat$'' to emphasize that the model distribution may differ from the true distribution $\px$ (i.e., the distribution from which the image $\vec{x}$ is actually drawn).
Under this prior model, the MMSE denoiser of $\vec{x}$ from $\vec{r}$ is
\begin{align}
\E_{\pxhat}\{\vec{x}|\vec{r}\} 
&\defn \fhatmmse(\vec{r})
\label{eq:fhatmmse} ,
\end{align}
and the likelihood of observing $\vec{r}$ is
\begin{align}
\prhat(\vec{r};\nu)
&\defn \int_{\Real^N} p(\vec{r}|\vec{x};\nu)\pxhat(\vec{x}) \deriv\vec{x}
\label{eq:prhat_def}\\
&= \int_{\Real^N} \mc{N}(\vec{r};\vec{x},\nu\vec{I}) \pxhat(\vec{x}) \deriv\vec{x}
\label{eq:pr} .
\end{align}
We will now define the \emph{Tweedie regularizer} (TR) as
\begin{align}
\rhoTR(\vec{r};\nu)
&\defn - \nu \ln \prhat(\vec{r};\nu)
\label{eq:rhoTR} .
\end{align}
As we now show, $\rhoTR(\cdot)$ has the desired property \eqref{desired}.

\begin{lemma}[Tweedie] \label{lem:gradTR}
For $\rhoTR(\vec{r};\nu)$ defined in \eqref{rhoTR}, 
\begin{align}
\nabla \rhoTR(\vec{r};\nu) = \vec{r}-\fhatmmse(\vec{r})
\label{eq:gradTR} ,
\end{align}
where $\fhatmmse(\cdot)$ is the MMSE denoiser from \eqref{fhatmmse}.
\end{lemma}
\begin{proof}
Equation \eqref{gradTR} is a direct consequence of a classical result known as Tweedie's formula \cite{Robbins:BSMSP:56,Efron:JASA:11}.
A short proof, from first principles, is now given for completeness.
\begin{align}
\lefteqn{
\frac{\partial}{\partial r_n} \rhoTR(\vec{r};\nu) 
= - \nu \frac{\partial}{\partial r_n} \ln \int_{\Real^N} \pxhat(\vec{x}) \mc{N}(\vec{r};\vec{x},\nu\vec{I}) \deriv\vec{x} }\\
&= - \frac{\nu \int_{\Real^N} \pxhat(\vec{x}) \frac{\partial}{\partial r_n} \mc{N}(\vec{r};\vec{x},\nu\vec{I}) \deriv\vec{x}}{\int_{\Real^N} \pxhat(\vec{x}) \mc{N}(\vec{r};\vec{x},\nu\vec{I}) \deriv\vec{x}} \\
&= \frac{\int_{\Real^N} \pxhat(\vec{x}) \mc{N}(\vec{r};\vec{x},\nu\vec{I}) (r_n-x_n) \deriv\vec{x}}{\int_{\Real^N} \pxhat(\vec{x}) \mc{N}(\vec{r};\vec{x},\nu\vec{I}) \deriv\vec{x}} \label{eq:deriv_gauss}\\
&= r_n - \int_{\Real^N} x_n \frac{\pxhat(\vec{x}) \mc{N}(\vec{r};\vec{x},\nu\vec{I}) }{\int_{\Real^N} \pxhat(\vec{x}') \mc{N}(\vec{r};\vec{x}',\nu\vec{I}) \deriv\vec{x}'} \deriv\vec{x} \qquad\\
&= r_n - \int_{\Real^N} x_n \,\pxgrhat(\vec{x}|\vec{r};\nu) \deriv\vec{x}\\
&= r_n - [\fhatmmse(\vec{r})]_n 
\label{eq:deriv_rhoTR},
\end{align}
where \eqref{deriv_gauss} used
$\frac{\partial}{\partial r_n} \mc{N}(\vec{r};\vec{x},\nu\vec{I})
 = \mc{N}(\vec{r};\vec{x},\nu\vec{I}) (x_n-r_n)/\nu$.
Stacking \eqref{deriv_rhoTR} for $n=1,\dots,N$ in a vector yields \eqref{gradTR}.
\end{proof}

Thus, if the TR regularizer $\rhoTR(\cdot;\nu)$ is used in the optimization problem \eqref{CRED}, then the solution $\hvec{x}$ must satisfy the fixed-point condition \eqref{fpRED} associated with the RED algorithms from \cite{Romano:JIS:17}, albeit with an MMSE-type denoiser.
This restriction will be removed using the SMD framework in \secref{SMD}.

It is interesting to note that the gradient property \eqref{gradTR} holds even for non-homogeneous $\fhatmmse(\cdot)$.
This generality is important in applications under which $\fhatmmse(\cdot)$ is known to lack LH.
For example, with a binary image $\vec{x}\in\{0,1\}^N$ modeled by $\pxhat(\vec{x})=\prod_{n=1}^N 0.5 (\delta(x_n)+\delta(x_n-1))$,
the MMSE denoiser takes the form
$[\fhatmmse(\vec{x})]_n = 0.5 + 0.5\tanh(x_n/\nu)$,
which is not LH.

\subsection{Tweedie Regularization as Kernel Density Estimation} \label{sec:kde}

We now show that TR arises naturally in the data-driven, non-parametric context through kernel-density estimation (KDE) \cite{Parzen:AMS:62}.

Recall that, in most imaging applications, the true prior $\px$ is unknown, as is the true MMSE denoiser $\fmmse(\cdot)$.
There are several ways to proceed.
One way is to design ``by hand'' an approximate prior $\pxhat$ that leads to a computationally efficient denoiser $\fhatmmse(\cdot)$.
But, because this denoiser is not MMSE for $\vec{x}\sim\px$, the performance of the resulting estimates $\hvec{x}$ will suffer relative to $\fmmse$.

Another way to proceed is to approximate the prior using a large corpus of training data $\{\vec{x}_t\}_{t=1}^T$.
To this end, an approximate prior could be formed using the empirical estimate
\begin{align}
\pxhat(\vec{x})
&= \frac{1}{T}\sum_{t=1}^T \delta(\vec{x}-\vec{x}_t)
\label{eq:px_emp} ,
\end{align}
but a more accurate match to the true prior $\px$ can be obtained using
\begin{align}
\pxsmooth(\vec{x};\nu) 
&= \frac{1}{T}\sum_{t=1}^T \mc{N}(\vec{x};\vec{x}_t,\nu\vec{I})
\label{eq:px_parzen} 
\end{align}
with appropriately chosen $\nu>0$, a technique known as kernel density estimation (KDE) or Parzen windowing \cite{Parzen:AMS:62}.
Note that if $\pxsmooth$ is used as a surrogate for $\px$, then the MAP optimization problem becomes
\begin{align}
\hvec{x}
&= \arg\min_{\vec{r}} \frac{1}{2\sigma^2}\|\vec{Ar}-\vec{y}\|^2 - \ln \pxsmooth(\vec{r};\nu) 
\label{eq:map_parzen} \\
&= \arg\min_{\vec{r}} \frac{1}{2\sigma^2}\|\vec{Ar}-\vec{y}\|^2 + \lambda
\rhoTR(\vec{r};\nu) \text{~for~}\lambda=\frac{1}{\nu} 
\label{eq:map_TR} ,
\end{align}
with $\rhoTR(\cdot;\nu)$ from \eqref{prhat_def}-\eqref{rhoTR} constructed using $\pxhat$ from \eqref{px_emp}.
In summary, TR arises naturally in the data-driven approach to image recovery when KDE is used to smooth the empirical prior.

\subsection{Score-Matching by Denoising} \label{sec:SMD} 

A limitation of the above TR framework is that it results in denoisers $\fhatmmse$ with symmetric Jacobians.
(Recall the discussion of MMSE denoisers in \secref{JS}.)
To justify the use of RED algorithms with non-symmetric Jacobians, we introduce the \emph{score-matching by denoising} (SMD) framework in this section.

Let us continue with the KDE-based MAP estimation problem \eqref{map_parzen}.
Note that $\hvec{x}$ from \eqref{map_parzen} zeros the gradient of the MAP optimization objective and thus obeys the fixed-point equation
\begin{align}
\frac{1}{\sigma^2}\vec{A}\tran(\vec{A}\hvec{x}-\vec{y}) - \nabla \ln \pxsmooth(\hvec{x};\nu) &= \vec{0}
\label{eq:fp_parzen} .
\end{align}
In principle, $\hvec{x}$ in \eqref{fp_parzen} could be found using gradient descent or similar techniques. 
However, computation of the gradient
\begin{align}
\nabla \ln \pxsmooth(\vec{r};\nu)
&= \frac{\nabla \pxsmooth(\vec{r};\nu)}{\pxsmooth(\vec{r};\nu)}
= \frac{\sum_{t=1}^T (\vec{x}_t-\vec{r}) 
        \mc{N}(\vec{r};\vec{x}_t,\nu\vec{I})}
   {\nu\sum_{t=1}^T \mc{N}(\vec{r};\vec{x}_t,\nu\vec{I})} 
\label{eq:score_parzen} 
\end{align}
is too expensive for the values of $T$ typically needed to generate a good image prior $\pxsmooth$.

A tractable alternative is suggested by the fact that 
\begin{align}
\nabla \ln \pxsmooth(\vec{r};\nu)
&= \frac{\fhatmmse(\vec{r}) -\vec{r}}{\nu} 
\label{eq:score_parzen2} \\
\text{for~} \fhatmmse(\vec{r}) 
&= \frac{\sum_{t=1}^T \vec{x}_t
        \mc{N}(\vec{r};\vec{x}_t,\nu\vec{I})}
        {\sum_{t=1}^T \mc{N}(\vec{r};\vec{x}_t,\nu\vec{I})} ,
\end{align}
where $\fhatmmse(\vec{r})$ is the MMSE estimator of $\vec{x}\sim\pxhat$ from $\vec{r}=\vec{x}+\mc{N}(\vec{0},\nu\vec{I})$.
In particular, if we can construct a good approximation to $\fhatmmse(\cdot)$ using a denoiser $\ftheta(\cdot)$ in a computationally efficient function class $\mc{F}\defn\{\vec{f}_{\vec{\theta}}: \vec{\theta}\in\vec{\Theta}\}$, then we can efficiently approximate the MAP problem \eqref{map_parzen}.

This approach can be formalized using the framework of \emph{score matching} \cite{Hyvarinen:JMLR:05}, which aims to approximate the ``score'' (i.e., the gradient of the log-prior) rather than the prior itself.
For example, suppose that we want to want to approximate the score $\nabla \ln \pxsmooth(\cdot;\nu)$.
For this, Hyv{\"a}rinen \cite{Hyvarinen:JMLR:05} suggested to first find the best mean-square fit among a set of computationally efficient functions $\vec{\psi}(\cdot;\vec{\theta})$, i.e., find
\begin{align}
\hvec{\theta}
&= \arg\min_{\vec{\theta}} \E_{\pxsmooth} \left\{\left\|
\vec{\psi}(\vec{x};\vec{\theta}) - \nabla \ln \pxsmooth(\vec{x};\nu) 
\right\|^2 \right\}
\label{eq:score_matching} ,
\end{align}
and then to approximate the score $\nabla \ln \pxsmooth(\cdot;\nu)$ by $\vec{\psi}(\cdot;\hvec{\theta})$. 
Later, in the context of denoising autoencoders, Vincent \cite{Vincent:NC:11} showed that if one chooses
\begin{align}
\vec{\psi}(\vec{x};\vec{\theta})
&= \frac{\ftheta(\vec{x})-\vec{x}}{\nu}
\label{eq:psi}
\end{align}
for some function $\ftheta(\cdot)\in\mc{F}$, then $\hvec{\theta}$ from \eqref{score_matching} can be equivalently written as
\begin{align}
\hvec{\theta}
&= \arg\min_{\vec{\theta}} 
\E_{\pxhat} \left\{ \left\| \ftheta\big(\vec{x}+\mc{N}(0,\nu\vec{I})\big) - \vec{x} \right\|^2\right\} .
\end{align}
In this case, $\fthetahat(\cdot)$ is the MSE-optimal denoiser, averaged over $\pxhat$ and constrained to the function class $\mc{F}$.

Note that the denoiser approximation error can be directly connected to the score-matching error as follows.
For any denoiser $\ftheta(\cdot)$ and any input $\vec{x}$,
\begin{align}
\lefteqn{ 
\|\ftheta(\vec{x})-\fhatmmse(\vec{x})\|^2 
}\nonumber\\
&=\nu^2\left\|\frac{\ftheta(\vec{x})-\vec{x}}{\nu} - \nabla\ln\pxsmooth(\vec{x};\nu)\right\|^2 
\label{eq:ferr1}\\
&=\nu^2\left\|\vec{\psi}(\vec{x};\vec{\theta}) - \nabla\ln\pxsmooth(\vec{x};\nu)\right\|^2
\label{eq:ferr2}
\end{align}
where \eqref{ferr1} follows from \eqref{score_parzen2}
and \eqref{ferr2} follows from \eqref{psi}.
Thus, matching the score is directly related to matching the MMSE denoiser.

Plugging the score approximation \eqref{psi} 
into the fixed-point condition \eqref{fp_parzen}, 
we get
\begin{align}
\frac{1}{\sigma^2}\vec{A}\tran(\vec{A}\hvec{x}-\vec{y}) + \lambda\big( \hvec{x} - \ftheta(\hvec{x}) \big)
&= \vec{0} \text{~for~}\lambda=\frac{1}{\nu}
\label{eq:fpRSS},
\end{align}
which matches the fixed-point condition \eqref{fpRED} of the RED algorithms from \cite{Romano:JIS:17}.
Here we emphasize that $\mc{F}$ may be constructed in such a way that $\ftheta(\cdot)$ has a non-symmetric Jacobian, which is the case for many state-of-the-art denoisers.
Also, $\vec{\theta}$ does not need to be optimized for \eqref{fpRSS} to hold.
Finally, $\pxhat$ need not be the empirical prior \eqref{px_emp}; it can be any chosen prior \cite{Vincent:NC:11}.
Thus, the score-matching-by-denoising (SMD) framework offers an explanation of the RED algorithms from \cite{Romano:JIS:17} that holds for generic denoisers $\ftheta(\cdot)$, whether or not they have symmetric Jacobians, are locally homogeneous, or MMSE.
Furthermore, it suggests a rationale for choosing 
the regularization weight $\lambda$ and, in the context of KDE,
the denoiser variance $\nu$.

\subsection{Relation to Existing Work}
Tweedie's formula \eqref{gradTR} has connections to Stein's Unbiased Risk Estimation (SURE) \cite{Stein:AS:81}, as discussed in, e.g., \cite[Thm.~2]{Luisier:Diss:10} and \cite[Eq.~(2.4)]{Raphan:NC:11}. 
SURE has been used for image denoising in, e.g., \cite{Blu:TIP:07}.
Tweedie's formula was also used in \cite{Bigdeli:17} to interpret autoencoding-based image priors.
In our work, Tweedie's forumula is used to provide an interpretation for the RED algorithms through the construction of the explicit regularizer \eqref{rhoTR} and the approximation of the resulting fixed-point equation \eqref{fp_parzen} via score matching. 

Recently, Alain and Bengio \cite{Alain:JMLR:14} studied the contractive auto-encoders, a type of autoencoder that minimizes squared reconstruction error plus a penalty that tries to make the autoencoder as simple as possible. 
While previous works such as \cite{Ranzato:NIPS:08} conjectured that such auto-encoders minimize an energy function, Alain and Bengio showed that they actually minimize the norm of a score (i.e., match a score to zero).
Furthermore, they showed that, when the coder and decoder do not share the same weights, it is not possible to define a valid energy function because the Jacobian of the reconstruction function is not symmetric.
The results in \cite{Alain:JMLR:14} parallel those in this paper, except that they focus on auto-encoders while we focus on variational image recovery.
Another small difference is that \cite{Alain:JMLR:14} uses the small-$\nu$ approximation
\begin{align}
\fhatmmse(\vec{x}) = \vec{x} + \nu \nabla \ln \pxhat(\vec{x}) + o(\nu),
\end{align}
whereas we use the exact (Tweedie's) relationship \eqref{gradTR}, i.e.,
\begin{align}
\fhatmmse(\vec{x}) = \vec{x} + \nu \nabla \ln \pxsmooth(\vec{x}) ,
\end{align}
where is $\pxsmooth$ the ``Gaussian blurred'' version of $\pxhat$ from \eqref{pr}.

\section{Fast RED Algorithms} \label{sec:algorithms}

In \cite{Romano:JIS:17}, Romano et al.\ proposed several ways 
to solve the fixed-point equation \eqref{fpRED}.
Throughout our paper, we have been referring to these methods as ``RED algorithms.''
In this section, we provide new interpretations of the RED-ADMM and RED-FP algorithms from \cite{Romano:JIS:17} and we propose new RED algorithms based on 
accelerated proximal gradient methods.

\subsection{RED-ADMM}

The ADMM approach was summarized in \algref{ADMM} for an arbitrary regularizer $\rho(\cdot)$.
To apply ADMM to RED, \lineref{ADMM_v_update} of \algref{ADMM}, known as the ``proximal update,'' must be specialized to the case where $\rho(\cdot)$ obeys \eqref{gradREDromano} for some denoiser $\vec{f}(\cdot)$.
To do this, Romano et al.\ \cite{Romano:JIS:17} proposed the following.
Because $\rho(\cdot)$ is differentiable, the proximal solution $\vec{v}_k$ must obey the fixed-point relationship
\begin{align}
\vec{0}
&= \lambda \nabla\rho(\vec{v}_k) + \beta (\vec{v}_k - \vec{x}_k - \vec{u}_{k-1}) \\
&= \lambda \big(\vec{v}_k-\vec{f}(\vec{v}_k)\big) + \beta (\vec{v}_k - \vec{x}_k - \vec{u}_{k-1}) \\
\Leftrightarrow~
\vec{v}_k
&= \frac{\lambda}{\lambda+\beta} \vec{f}(\vec{v}_k) + \frac{\beta}{\lambda+\beta} (\vec{x}_k + \vec{u}_{k-1}) 
\label{eq:proxRED} .
\end{align}
An approximation to $\vec{v}_k$ can thus be obtained by iterating 
\begin{align}
\vec{z}_{i}
&= \frac{\lambda}{\lambda+\beta} \vec{f}(\vec{z}_{i-1}) + \frac{\beta}{\lambda+\beta} (\vec{x}_k + \vec{u}_{k-1}) 
\label{eq:proxREDinexact} 
\end{align}
over $i=1,\dots,I$ with sufficiently large $I$, initialized at $\vec{z}_0=\vec{v}_{k-1}$. 
This procedure is detailed in lines~\ref{line:RED_ADMM_z_init}-\ref{line:RED_ADMM_z_end} of \algref{RED_ADMM}.
The overall algorithm is known as RED-ADMM. 

\begin{algorithm}[t]
\caption{RED-ADMM with $I$ Inner Iterations\cite{Romano:JIS:17}}
\begin{algorithmic}[1] \label{alg:RED_ADMM}
\REQUIRE{$\ell(\cdot;\vec{y}),\vec{f}(\cdot),
          \beta,\lambda,
          \vec{v}_0,\vec{u}_0,K$,
          and $I$} \label{line:RED_ADMM_init}
\FOR{$k = 1,2,\dots,K$}
 	\STATE{$\vec{x}_{k}=\arg\min_{\vec{x}}\{\ell(\vec{x};\vec{y})+\frac{\beta}{2}\norm{\vec{x}-\vec{v}_{k-1}+\vec{u}_{k-1}}^2\}$}\label{line:RED_ADMM_x_update}
 	\STATE{$\vec{z}_0 = \vec{v}_{k-1}$}\label{line:RED_ADMM_z_init}
 	\FOR{$i=1,2,\dots,I$}
 		\STATE{$\vec{z}_i=\frac{\lambda}{\lambda+\beta} \vec{f}(\vec{z}_{i-1})+\frac{\beta}{\lambda+\beta}(\vec{x}_k+\vec{u}_{k-1})$}\label{line:RED_ADMM_z_update}
 	\ENDFOR \label{line:RED_ADMM_z_end}
 	\STATE{$\vec{v}_k = \vec{z}_{I}$}\label{line:RED_ADMM_v_update}
 	\STATE{$\vec{u}_k = \vec{u}_{k-1}+\vec{x}_k-\vec{v}_k$}\label{line:RED_ADMM_u_update}
\ENDFOR
\STATE{Return $\vec{x}_K$}
\end{algorithmic}
\end{algorithm}
 
\subsection{Inexact RED-ADMM}

\algref{RED_ADMM} gives a faithful implementation of ADMM when the number of inner iterations, $I$, is large.
But using many inner iterations may be impractical when the denoiser is computationally expensive, as in the case of BM3D or TNRD.
Furthermore, the use of many inner iterations may not be necessary.

For example, \figref{ADMM_time_log} plots PSNR trajectories versus runtime for TNRD-based RED-ADMM with $I=1,2,3,4$ inner iterations. 
For this experiment, we used the deblurring task described in \secref{compare}, but similar behaviors can be observed in other applications of RED.
\Figref{ADMM_time_log} suggests that $I=1$ inner iterations gives the fastest convergence. 
Note that \cite{Romano:JIS:17} also used $I=1$ when implementing RED-ADMM.

\putFrag{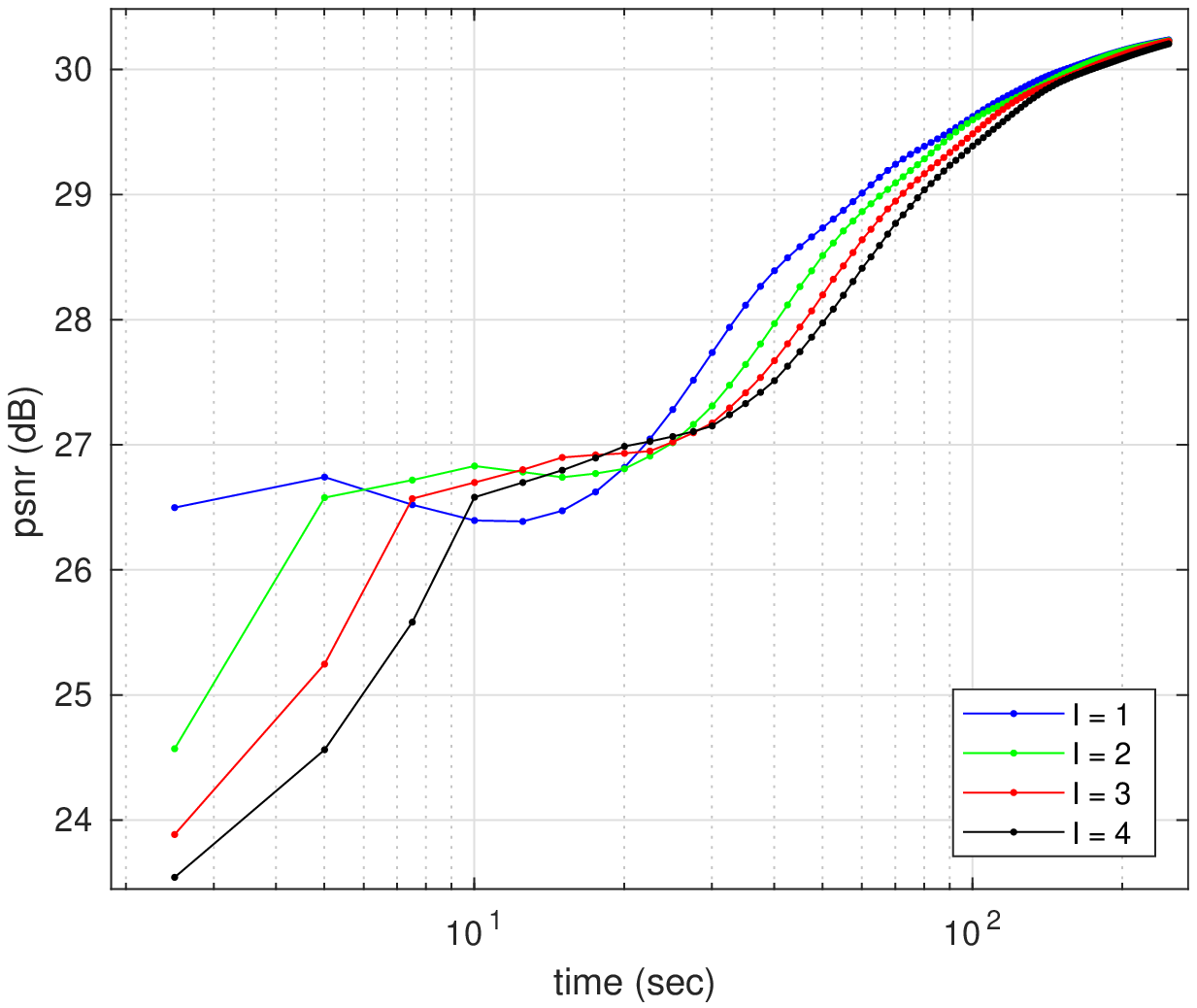}
	{PSNR versus runtime for RED-ADMM with TNRD denoising and $I$ inner iterations.}
	{\figsize}
  	{\psfrag{time (sec)}[t][t][0.7]{\sf time (sec)}
	 \psfrag{psnr (dB)}[b][b][0.7]{\sf PSNR}
	 \psfrag{I = 1}[l][l][0.6]{\sf $I=1$}
	 \psfrag{I = 2}[l][l][0.6]{\sf $I=2$}
	 \psfrag{I = 3}[l][l][0.6]{\sf $I=3$}
	 \psfrag{I = 4}[l][l][0.6]{\sf $I=4$} }
        {trim=20pt 10pt 20pt 20pt}
 
\begin{algorithm}[t]
\caption{RED-ADMM with $I=1$}
\begin{algorithmic}[1] \label{alg:RED_ADMM_inexact}
\REQUIRE{$\ell(\cdot;\vec{y}),\vec{f}(\cdot),
          \beta,\lambda,
          \vec{v}_0,\vec{u}_0$,
          and $K$} \label{line:ADMM2_init}
\FOR{$k = 1,2,\dots,K$}
 	\STATE{$\vec{x}_{k}=\arg\min_{\vec{x}}\{\ell(\vec{x};\vec{y})+\frac{\beta}{2}\norm{\vec{x}-\vec{v}_{k-1}+\vec{u}_{k-1}}^2\}$}\label{line:ADMM2_x_update}
 	\STATE{$\vec{v}_k = \frac{\lambda}{\lambda+\beta} \vec{f}(\vec{v}_{k-1})+\frac{\beta}{\lambda+\beta}(\vec{x}_k+\vec{u}_{k-1})$}\label{line:ADMM2_v_update}
 	\STATE{$\vec{u}_k = \vec{u}_{k-1}+\vec{x}_k-\vec{v}_k$}\label{line:ADMM2_u_update}
\ENDFOR
\STATE{Return $\vec{x}_K$}
\end{algorithmic}
\end{algorithm}
 
With $I=1$ inner iterations, RED-ADMM simplifies down to the 3-step iteration summarized in \algref{RED_ADMM_inexact}. 
Since \algref{RED_ADMM_inexact} looks quite different than standard ADMM (recall \algref{ADMM}), one might wonder whether there exists another interpretation of \algref{RED_ADMM_inexact}.
Noting that \lineref{ADMM2_v_update} can be rewritten as
\begin{align}
\vec{v}_k 
&=\vec{v}_{k-1}-\frac{1}{\lambda+\beta}\big[\lambda
\nabla\rho(\vec{v}_{k-1})
+\beta(\vec{v}_{k-1}-\vec{x}_k-\vec{u}_{k-1})\big] \\
&=\vec{v}_{k-1}-\frac{1}{\lambda+\beta} \nabla \left[ \lambda \rho(\vec{v}) 
        + \frac{\beta}{2}\|\vec{v} - \vec{x}_k-\vec{u}_{k-1}\|^2 \right]
        _{\vec{v}=\vec{v}_{k-1}}
\end{align}
we see that the $I=1$ version of inexact RED-ADMM replaces the proximal step with a gradient-descent step under stepsize $1/(\lambda+\beta)$.
Thus the algorithm is reminiscent of the proximal gradient (PG) algorithm \cite{Beck:Chap:09,Combettes:Chap:11}.
We will discuss PG further in the sequel.

\subsection{Majorization-Minimization and Proximal-Gradient RED} \label{sec:MM}

We now propose a proximal-gradient approach inspired by \emph{majorization minimization} (MM) \cite{Sun:TSP:17}.
As proposed in \cite{Figueiredo:TIP:07}, we use a quadratic upper-bound,
\begin{align}
\bar{\rho}(\vec{x};\vec{x}_k)
&\defn \rho(\vec{x}_{k})
        + [\nabla\rho(\vec{x}_{k})]\tran\big(\vec{x}-\vec{x}_{k}\big) 
        + \frac{L}{2}\norm{\vec{x}-\vec{x}_{k}}_2^2 
\label{eq:rhobound} ,
\end{align}
on the regularizer $\rho(\vec{x})$, in place of $\rho(\vec{x})$ itself,
at the $k$th algorithm iteration.
Note that if $\rho(\cdot)$ is convex and $\nabla\rho(\cdot)$ is $L_\rho$-Lipschitz, then $\bar{\rho}(\vec{x};\vec{x}_k)$ ``majorizes'' $\rho(\vec{x})$ at $\vec{x}_k$ when $L\geq L_\rho$, i.e.,
\begin{align}
\bar{\rho}(\vec{x};\vec{x}_k)
&\geq \rho(\vec{x})~\forall \vec{x}\in\mc{X}  \\
\bar{\rho}(\vec{x}_k;\vec{x}_k)
&= \rho(\vec{x}_k) .
\end{align}
The majorized objective can then be minimized using the \emph{proximal gradient} (PG) algorithm \cite{Beck:Chap:09,Combettes:Chap:11} (also known as forward-backward splitting) as follows.
From \eqref{rhobound}, note that the majorized objective can be written as
\begin{eqnarray}
\lefteqn{
\ell(\vec{x};\vec{y})+\lambda\bar{\rho}(\vec{x};\vec{x}_k)
}\nonumber\\
&=& \ell(\vec{x};\vec{y})+ \frac{\lambda L}{2}\left\| 
       \vec{x} - \left(\vec{x}_k- \frac{1}{L}\nabla\rho(\vec{x}_k)\right) 
       \right\|^2 + \text{const} \\
&=& \ell(\vec{x};\vec{y})+ \frac{\lambda L}{2}\bigg\| \vec{x} - 
        \underbrace{
   \bigg(\vec{x}_k- \frac{1}{L}\big(\vec{x}_k-\vec{f}(\vec{x}_k)\big) \bigg) 
        }_{\displaystyle \defn \vec{v}_k}
       \bigg\|^2 + \text{const} ,
\nonumber\\[-4.5mm]
\label{eq:objPG}
\end{eqnarray}
where \eqref{objPG} follows from assuming \eqref{desired}, which is the basis for all RED algorithms.
The RED-PG algorithm then alternately updates $\vec{v}_k$ as per the gradient step in \eqref{objPG} and updates $\vec{x}_{k+1}$ according to the proximal step 
\begin{align}
\vec{x}_{k+1}
&= \arg\min_{\vec{x}} \left\{\ell(\vec{x};\vec{y})+\frac{\lambda L}{2}\norm{\vec{x}-\vec{v}_k}^2 \right\} ,
\end{align}
as summarized in \algref{RED_PG}.
Convergence is guaranteed if $L\geq L_\rho$; see \cite{Beck:Chap:09,Combettes:Chap:11} for details.

 \begin{algorithm}[t]
 \caption{RED-PG Algorithm}
 \begin{algorithmic}[1] \label{alg:RED_PG}
 \REQUIRE{$\ell(\cdot;\vec{y}),\vec{f}(\cdot),
          \lambda,
          \vec{v}_0, L>0$,
          and $K$} \label{line:PG_init}
 \FOR{$k = 1,2,\dots,K$}
 	\STATE{$\vec{x}_{k}=\arg\min_{\vec{x}}\{\ell(\vec{x};\vec{y})+\frac{\lambda L}{2}\norm{\vec{x}-\vec{v}_{k-1}}^2\}$} \label{line:PG_x_update}
 	\STATE{$\vec{v}_k = \frac{1}{L}\vec{f}(\vec{x}_k)-\frac{1-L}{L}\vec{x}_{k}$} \label{line:PG_v_update}
 \ENDFOR
 \STATE{Return $\vec{x}_K$}
 \end{algorithmic}
 \end{algorithm}

We now show that RED-PG with $L=1$ is identical to the ``fixed point'' (FP) RED algorithm proposed in \cite{Romano:JIS:17}.
First, notice from \algref{RED_PG} that $\vec{v}_k=\vec{f}(\vec{x}_k)$ when $L=1$, in which case
\begin{align}
\vec{x}_k
&= \arg\min_{\vec{x}} \left\{ \ell(\vec{x};\vec{y})+\frac{\lambda}{2}\norm{\vec{x}-\vec{f}(\vec{x}_{k-1})}^2 \right\} 
\label{eq:RED_FP}.
\end{align}
For the quadratic loss
$\ell(\vec{x};\vec{y})=\frac{1}{2\sigma^2}\|\vec{Ax}-\vec{y}\|^2$,
\eqref{RED_FP} becomes
\begin{align}
\vec{x}_k
&= \arg\min_{\vec{x}} \left\{ \frac{1}{2\sigma^2}\|\vec{Ax}-\vec{y}\|^2
        +\frac{\lambda}{2}\norm{\vec{x}-\vec{f}(\vec{x}_{k-1})}^2 \right\} \\
&= \Big(\frac{1}{\sigma^2}\vec{A}\tran\vec{A}+\lambda\vec{I}\Big)^{-1}
        \Big(\frac{1}{\sigma^2}\vec{A}\tran\vec{y} + \lambda\vec{f}(\vec{x}_{k-1})\Big)
\label{eq:RED_FPQ},
\end{align}
which is exactly the RED-FP update \cite[(37)]{Romano:JIS:17}.
Thus, \eqref{RED_FP} generalizes \cite[(37)]{Romano:JIS:17} to possibly non-quadratic\footnote{%
The extension to non-quadratic loss is important for applications like phase-retrieval, where RED has been successfully applied \cite{Metzler:ICML:18}.}
loss $\ell(\cdot;\vec{y})$,
and RED-PG generalizes RED-FP to arbitrary $L>0$.
More importantly, the PG framework facilitates algorithmic acceleration, as we describe below.

The RED-PG and inexact RED-ADMM-$I\!=\!1$ algorithms show interesting similarities: both alternate a proximal update on the loss with a gradient update on the regularization, where the latter term manifests as a convex combination between the denoiser output and another term.
The difference is that RED-ADMM-$I\!=\!1$ includes an extra state variable, $\vec{u}_k$.
The experiments in \secref{compare} suggest that this extra state variable is not necessarily advantageous.

\subsection{Dynamic RED-PG} \label{sec:DPG}

Recalling from \eqref{objPG} that $1/L$ acts as a stepsize in the PG gradient step, it may be possible to speed up PG by decreasing $L$, although making $L$ too small can prevent convergence.
If $\rho(\cdot)$ was known, then a line search could be used, at each iteration $k$,
to find the smallest value of $L$ that guarantees the majorization of $\rho(\vec{x})$ by $\bar{\rho}(\vec{x};\vec{x}_k)$ \cite{Beck:Chap:09}. 
However, with a non-LH or non-JS denoiser, it is not possible to evaluate $\rho(\cdot)$, preventing such a line search.

We thus propose to vary $L_k$ (i.e., the value of $L$ at iteration $k$) according to a fixed schedule.
In particular, we propose to select $L_0$ and $L_\infty$, and smoothly interpolate between them at intermediate iterations $k$.
One interpolation scheme that works well in practice is summarized in \lineref{DPG_Lk_update} of \algref{RED_DPG}.
We refer to this approach as ``dynamic PG'' (DPG).
The numerical experiments in \secref{compare} suggest that, with appropriate selection of $L_0$ and $L_\infty$, RED-DPG can be significantly faster than RED-FP.

 \begin{algorithm}[t]
 \caption{RED-DPG Algorithm}
 \begin{algorithmic}[1] \label{alg:RED_DPG}
 \REQUIRE{$\ell(\cdot;\vec{y}),\vec{f}(\cdot),
          \lambda,
          \vec{v}_0, L_0>0, L_\infty>0$,
          and $K$} \label{line:DPG_init}
 \FOR{$k = 1,2,\dots,K$}
 	\STATE{$\vec{x}_{k}=\arg\min_{\vec{x}}\{\ell(\vec{x};\vec{y})+\frac{\lambda L_{k-1}}{2}\norm{\vec{x}-\vec{v}_{k-1}}^2\}$}\label{line:DPG_x_update}
        \STATE{$L_k = \big(\frac{1}{L_\infty} + (\frac{1}{L_0}-\frac{1}{L_\infty})\frac{1}{\sqrt{k+1}}\big)^{-1}$}\label{line:DPG_Lk_update}
 	\STATE{$\vec{v}_k = \frac{1}{L_k}\vec{f}(\vec{x}_k)-\frac{1-L_k}{L_k}\vec{x}_{k}$}\label{line:DPG_v_update}
 \ENDFOR
 \STATE{Return $\vec{x}_K$}
 \end{algorithmic}
 \end{algorithm}
 
\subsection{Accelerated RED-PG} \label{sec:APG}

Another well-known approach to speeding up PG is to apply momentum to the $\vec{v}_k$ term in \algref{RED_PG} \cite{Beck:Chap:09}, often known as ``acceleration.''
An accelerated PG (APG) approach to RED is detailed in \algref{RED_APG}. 
There, the momentum in \lineref{APG_z_update} takes the same form as in FISTA \cite{Beck:JIS:09}.
The numerical experiments in \secref{compare} suggest that RED-APG is the fastest among the RED algorithms discussed above.

 \begin{algorithm}[t]
 \caption{RED-APG Algorithm}
 \begin{algorithmic}[1] \label{alg:RED_APG}
 \REQUIRE{$\ell(\cdot;\vec{y}),\vec{f}(\cdot),
          \lambda,
          \vec{v}_0, L>0$,
          and $K$} \label{line:APG_init}
 \STATE{$t_0=1$}
 \FOR{$k = 1,2,\dots,K$}
 	\STATE{$\vec{x}_{k}=\arg\min_{\vec{x}}\{\ell(\vec{x};\vec{y})+\frac{\lambda L}{2}\norm{\vec{x}-\vec{v}_{k-1}}^2\}$}\label{line:APG_x_update}
        \STATE{$t_k=\frac{1+\sqrt{1+4t_{k-1}^2}}{2}$}\label{line:APG_t_update}
 	\STATE{$\vec{z}_k = \vec{x}_k + \frac{t_{k-1}-1}{t_k}(\vec{x}_k-\vec{x}_{k-1})$}\label{line:APG_z_update}
 	\STATE{$\vec{v}_k = \frac{1}{L}\vec{f}(\vec{z}_k)-\frac{1-L}{L}\vec{z}_k $}\label{line:APG_v_update}
 \ENDFOR
 \STATE{Return $\vec{x}_K$}
 \end{algorithmic}
 \end{algorithm}

By leveraging the principle of vector extrapolation (VE) \cite{Sidi:Book:17},
a different approach to accelerating RED algorithms was recently proposed in \cite{Hong:18}.
Algorithmically, the approach in \cite{Hong:18} is much more complicated than the PG-DPG and PG-APG methods proposed above.
In fact, we have been unable to arrive at an implementation of \cite{Hong:18} that reproduces the results in that paper, and the authors have not been willing to share their implementation with us.
Thus, we cannot comment further on the difference in performance between our PG-DPG and PG-APG schemes and the one in \cite{Hong:18}.

\subsection{Convergence of RED-PG}
Recalling \thmref{impossible}, the RED algorithms do not minimize an explicit cost function but rather seek fixed points of \eqref{fpRED}.
Therefore, it is important to know whether they actually converge to any one fixed point.
Below, we use the theory of non-expansive and $\alpha$-averaged operators to establish the convergence of RED-PG to a fixed point under certain conditions.

First, an operator $\vec{B}(\cdot)$ is said to be \emph{non-expansive} if its Lipschitz constant is at most $1$ \cite{Bauschke:Book:11}.
Next, for $\alpha\in (0,1)$, an operator $\vec{P}(\cdot)$ is said to be \emph{$\alpha$-averaged} if
\begin{align}
\vec{P}(\vec{x}) = \alpha \vec{B} (\vec{x}) + (1 - \alpha) \vec{x} 
\label{eq:alpha_average_def} 
\end{align}
for some non-expansive $\vec{B}(\cdot)$.
Furthermore, if $\vec{P}_1$ and $\vec{P}_2$ are $\alpha_1$ and $\alpha_2$-averaged, respectively, then \cite[Prop.~4.32]{Bauschke:Book:11} establishes that the composition $\vec{P}_2 \circ \vec{P}_1$ is $\alpha$-averaged with
\begin{align}
\alpha = \frac{2}{1+\frac{1}{\max\left\{\alpha_1,\alpha_2\right\}}} 
\label{eq:alpha_composition} .
\end{align}

Recalling RED-PG from \algref{RED_PG}, let us define an operator called $\vec{T}(\cdot)$ that summarizes one algorithm iteration:
\begin{align}
\lefteqn{ \vec{T}(\vec{x}) }\nonumber\\
&\defn \arg\min_{\vec{z}} \Big\{ \ell(\vec{z};\vec{y}) + \tfrac{\lambda L}{2}\big\|\vec{z} - \big(\tfrac{1}{L} \vec{f}(\vec{x})- \tfrac{1-L}{L}\vec{x}\big)\big\|^2 \Big\}
\label{eq:T_minimization} \\
&= \prox_{\ell/(\lambda L)}\big(\tfrac{1}{L}(\vec{f}(\vec{x})-(1-L)\vec{x})\big) \label{eq:T_prox}
\end{align}

\begin{lemma} \label{lem:T_alpha}
If $\ell(\cdot)$ is proper, convex, and continuous; $\vec{f}(\cdot)$ is non-expansive; and $L>1$, then $\vec{T}(\cdot)$ from \eqref{T_prox} is $\alpha$-averaged with $\alpha = \max\{\tfrac{2}{1+L},\tfrac{2}{3}\}$. 
\end{lemma}
\begin{proof}
First, because $\ell(\cdot)$ is proper, convex, and continuous, we know that the proximal operator $\prox_{\ell/(\lambda L)}(\cdot)$ is $\alpha$-averaged with $\alpha=1/2$ \cite{Bauschke:Book:11}. 
Then, by definition, $\frac{1}{L}\vec{f}(\vec{z})-\frac{1-L}{L}\vec{z}$ is $\alpha$-averaged with $\alpha=1/L$.
From \eqref{T_prox}, $\vec{T}(\cdot)$ is the composition of these two $\alpha$-averaged operators, and so from \eqref{alpha_composition} we have that $\vec{T}(\cdot)$ is $\alpha$-averaged with
$\alpha = \max\{\frac{2}{1+L},\frac{2}{3}\}$.
\end{proof}

With \lemref{T_alpha}, we can prove the convergence of RED-PG.

\begin{theorem} \label{thm:mann}
If $\ell(\cdot)$ is proper, convex, and continuous; $\vec{f}(\cdot)$ is non-expansive; $L>1$; and $\vec{T}(\cdot)$ from \eqref{T_prox} has at least one fixed point, then RED-PG converges. 
\end{theorem}
\begin{proof}
From \eqref{T_prox}, we have that \algref{RED_PG} is equivalent to 
\begin{align}
\vec{x}_{k+1} &= \vec{T}(\vec{x}_k) \label{eq:mann_1}\\
&= \alpha\vec{B}(\vec{x}_k) + (1-\alpha)\vec{x}_k \label{eq:mann_2}
\end{align}
where $\vec{B}(\cdot)$ is an implicit non-expansive operator that must exist under the definition of $\alpha$-averaged operators from \eqref{alpha_average_def}.
The iteration \eqref{mann_2} can be recognized as a Mann iteration \cite{Parikh:FTO:13}, since $\alpha\in(0,1)$.
Thus, from \cite[Thm.~5.14]{Bauschke:Book:11}, $\{\vec{x}_k\}$ is a convergent sequence, in that there exists a fixed point $\vec{x}_\star\in\Real^N$ such that $\lim_{k\to\infty} \|\vec{x}_k - \vec{x}_\star\| = 0$. 
\end{proof}

We note that similar Mann-based techniques were used in \cite{Buzzard:JIS:18,Sun:18} to prove the convergence of PnP-based algorithms.
Also, we conjecture that similar techniques may be used to prove the convergence of other RED algorithms, but we leave the details to future work.
Experiments in \secref{compare} numerically study the convergence behavior of several RED algorithms with different image denoisers $\vec{f}(\cdot)$. 

\subsection{Algorithm Comparison: Image Deblurring} \label{sec:compare}

\putFrag{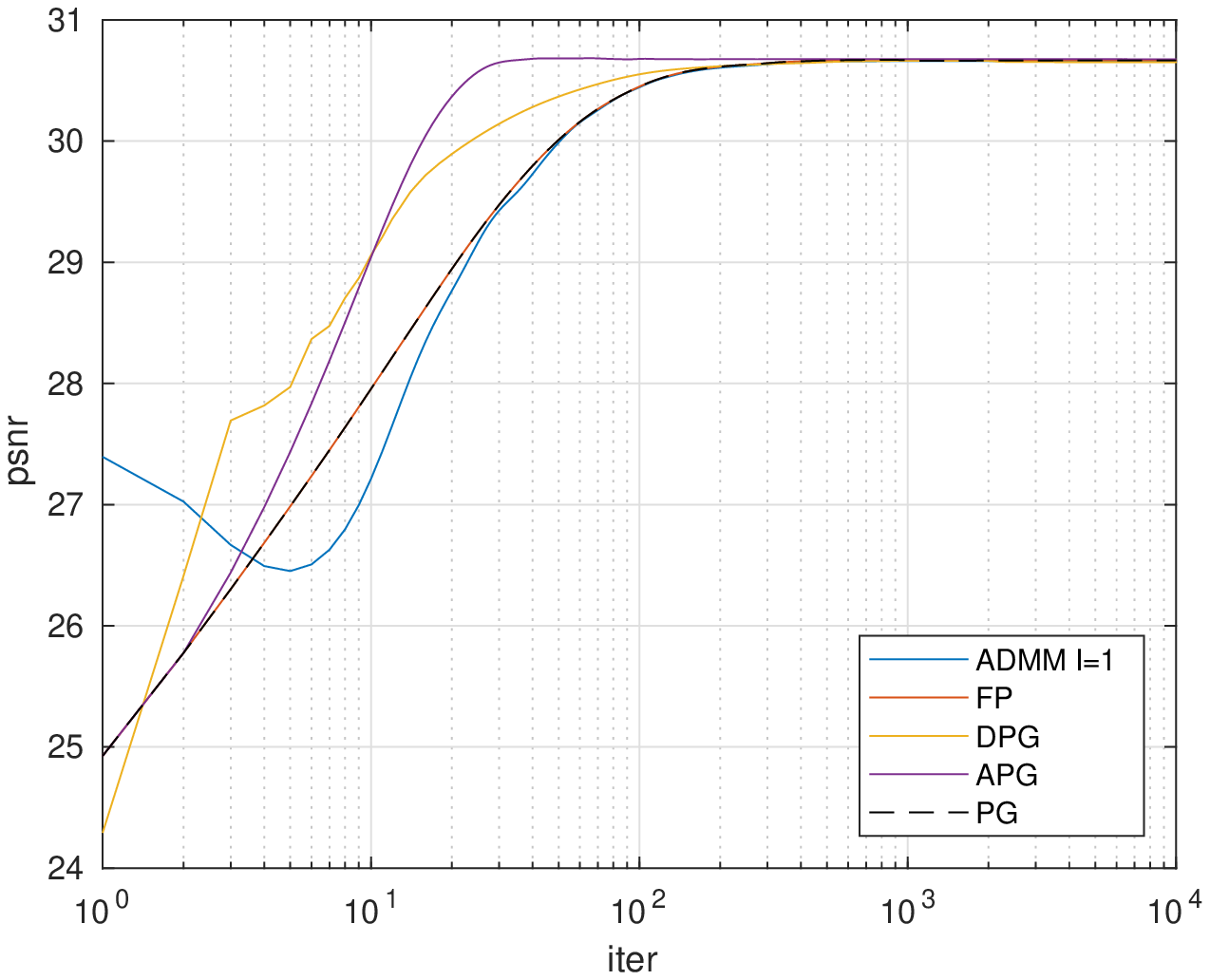}
 	{PSNR versus iteration for RED algorithms with TNRD denoising when deblurring the starfish.}
 	{\figsize}
        {\newcommand{\sz}{0.6}
 	 \psfrag{iter}[t][t][0.7]{\sf iteration}
	 \psfrag{time (sec)}[t][t][0.7]{\sf time (sec)}
	 \psfrag{psnr}[B][B][0.7]{\sf PSNR}  
	 \psfrag{ADMM I=1}[l][l][\sz]{\sf ADMM-$I\!=\!1$}
	 \psfrag{PRS-ADMM}[l][l][\sz]{\sf PRS-$I\!=\!1$}
	 \psfrag{FP}[l][l][\sz]{\sf FP}
	 \psfrag{GEC}[l][l][\sz]{\sf GEC-$I\!=\!1$}
	 \psfrag{PG-DP}[l][l][\sz]{\sf DPG}
	 \psfrag{PG-FISTA}[l][l][\sz]{\sf APG}
	 \psfrag{DPG}[l][l][\sz]{\sf DPG}
	 \psfrag{APG}[l][l][\sz]{\sf APG}
	 \psfrag{PG}[l][l][\sz]{\sf PG} }
        {trim=20pt 10pt 20pt 20pt}
        
\putFrag{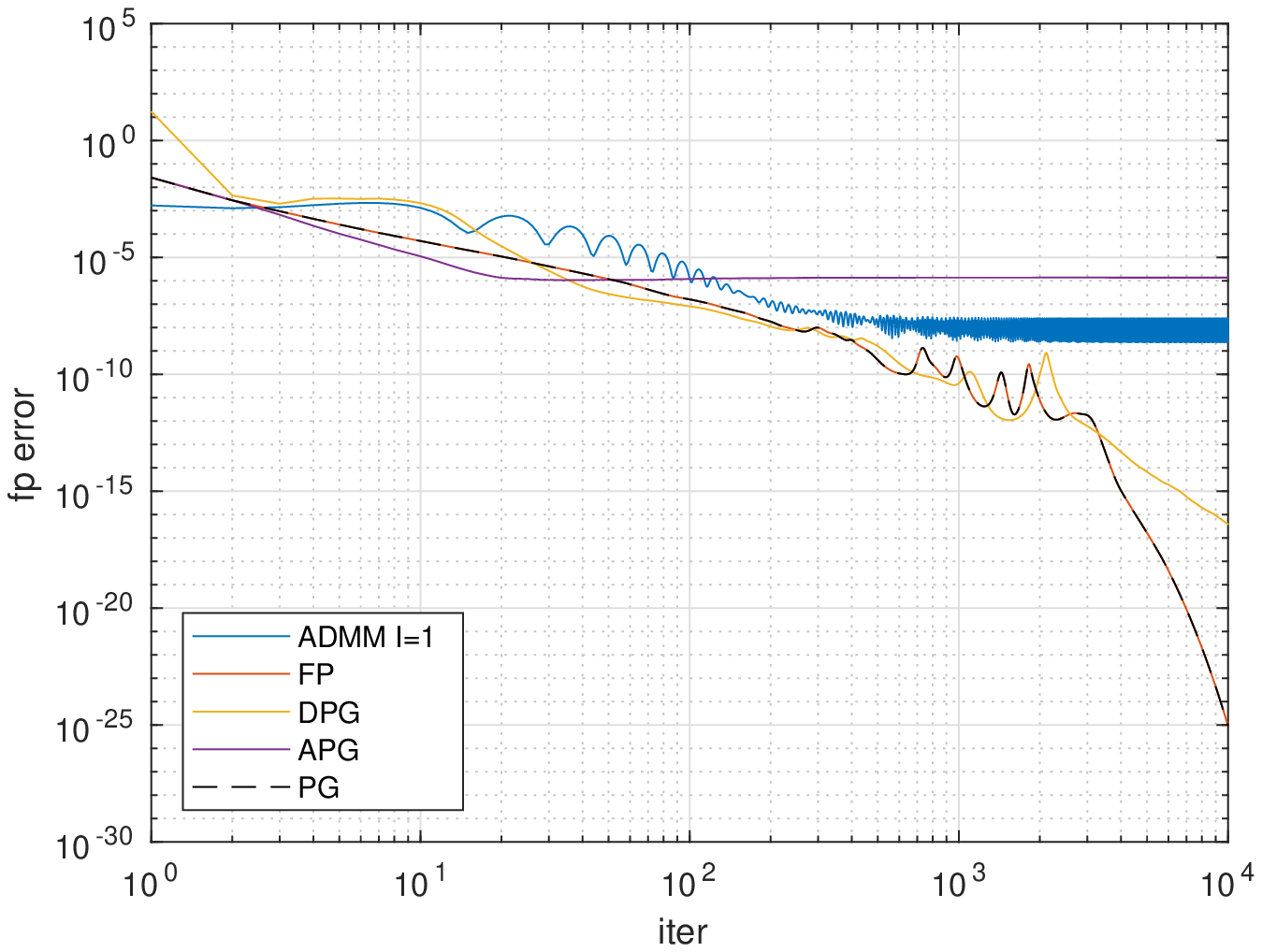}
 	{Fixed-point error versus iteration for RED algorithms with TNRD denoising when deblurring the starfish.}
 	{\figsize}
 	 {\newcommand{\sz}{0.6}
 	 \psfrag{iter}[t][t][0.7]{\sf iteration}
	 \psfrag{ADMM I=1}[l][l][\sz]{\sf ADMM-$I\!=\!1$}
	 \psfrag{PG-DP}[l][l][\sz]{\sf DPG}
	 \psfrag{FP}[l][l][\sz]{\sf FP}
	 \psfrag{DPG}[l][l][\sz]{\sf DPG}
	 \psfrag{APG}[l][l][\sz]{\sf APG}
	 \psfrag{PG}[l][l][\sz]{\sf PG}
         \psfrag{fp error}[B][B][0.7]{$\frac{1}{N}\big\|\frac{1}{\sigma^2}\vec{A}^H(\vec{Ax}_k-\vec{y}) + \lambda(\vec{x}_k-\vec{f}(\vec{x}_k))\big\|^2$}}
     	{trim=20pt 10pt 20pt 20pt}

\putFrag{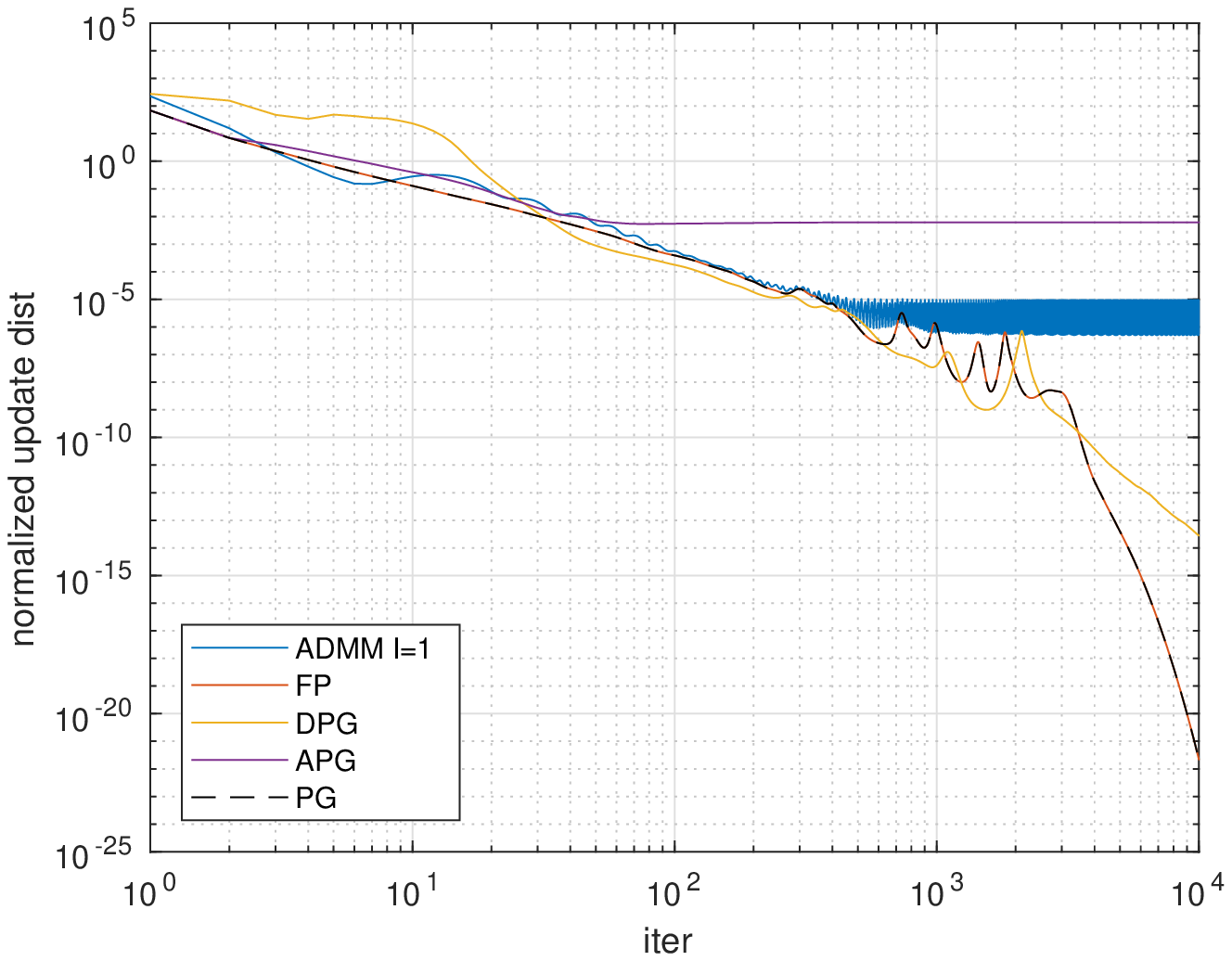}
 	{Update distance versus iteration for RED algorithms with TNRD denoising when deblurring the starfish.}
 	{\figsize}
        {\newcommand{\sz}{0.6}
 	 \psfrag{iter}[t][t][0.7]{\sf iteration}
	 \psfrag{time (sec)}[t][t][0.7]{\sf time (sec)}
	 \psfrag{ADMM I=1}[l][l][\sz]{\sf ADMM-$I\!=\!1$}
	 \psfrag{PRS-ADMM}[l][l][\sz]{\sf PRS-$I\!=\!1$}
	 \psfrag{FP}[l][l][\sz]{\sf FP}
	 \psfrag{GEC}[l][l][\sz]{\sf GEC-$I\!=\!1$}
	 \psfrag{PG-FISTA}[l][l][\sz]{\sf APG}
	 \psfrag{DPG}[l][l][\sz]{\sf DPG}
	 \psfrag{APG}[l][l][\sz]{\sf APG}
	 \psfrag{PG}[l][l][\sz]{\sf PG}
	 \psfrag{normalized update dist}[B][B][0.7]{\sf $\frac{1}{N}\|\vec{x}_k-\vec{x}_{k-1}\|^2$} }
        {trim=20pt 10pt 20pt 20pt}        
        
We now compare the performance of the RED algorithms discussed above (i.e., inexact ADMM, FP, DPG, APG, and PG) on the image deblurring problem considered in \cite[Sec.~6.1]{Romano:JIS:17}.
For these experiments, the measurements $\vec{y}$ were constructed using a $9\times 9$ uniform blur kernel for $\vec{A}$ and using AWGN with variance $\sigma^2=2$.
As stated earlier, the image $\vec{x}$ is normalized to have pixel intensities in the range $[0,255]$.

For the first experiment, we used the TNRD denoiser.
The various algorithmic parameters were chosen based on the recommendations in \cite{Romano:JIS:17}:
the regularization weight was $\lambda=0.02$, 
the ADMM penalty parameter was $\beta=0.001$, 
and
the noise variance assumed by the denoiser was $\nu=3.25^2$. 
The proximal step on $\ell(\vec{x};\vec{y})$, given 
in \eqref{RED_FPQ}, was implemented with an FFT.
For RED-DPG we used\footnote{Matlab code for these experiments is available at \url{http://www2.ece.ohio-state.edu/~schniter/RED/index.html}.}
$L_0=0.2$ and $L_\infty=2$, for RED-APG we used $L=1$, and for RED-PG we used $L=1.01$ since \thmref{mann} motivates $L>1$.

\Figref{psnr_tnrd} shows 
$$
\text{PSNR}_k \defn -10\log_{10}\left(\frac{1}{N 256^2}\|\vec{x}-\hvec{x}_k\|^2\right) 
$$
versus iteration $k$ for the starfish test image.
In the figure, the proposed RED-DPG and RED-APG algorithms appear significantly faster than the RED-FP and RED-ADMM-$I\!=\!1$ algorithms proposed in \cite{Romano:JIS:17}.
For example, RED-APG reaches PSNR~$=30$ in $15$ iterations
whereas RED-FP and inexact RED-ADMM-$I=1$ take about
$50$ iterations.

\Figref{fp_tnrd} shows the fixed-point error
\begin{align*}
\frac{1}{N}\bigg\|\frac{1}{\sigma^2}\vec{A}^H(\vec{Ax}_k-\vec{y}) + \lambda(\vec{x}_k-\vec{f}(\vec{x}_k))\bigg\|^2
\end{align*}
verus iteration $k$.
All but the RED-APG and RED-ADMM algorithms appear to converge to the solution set of the fixed-point equation \eqref{fpRED}.
The RED-APG and RED-ADMM algorithms appear to approximately satisfy the fixed-point equation \eqref{fpRED}, but not exactly satisfy \eqref{fpRED}, since the fixed-point error does not decay to zero.

\Figref{dist_tnrd} shows the update distance $\frac{1}{N}\|\vec{x}_k-\vec{x}_{k-1}\|^2$ vs.\ iteration $k$ for the algorithms under test.
For most algorithms, the update distance appears to be converging to zero, but for RED-APG and RED-ADMM it does not.
This suggests that the RED-APG and RED-ADMM algorithms are converging to a limit cycle rather than a unique limit point.

\putFrag{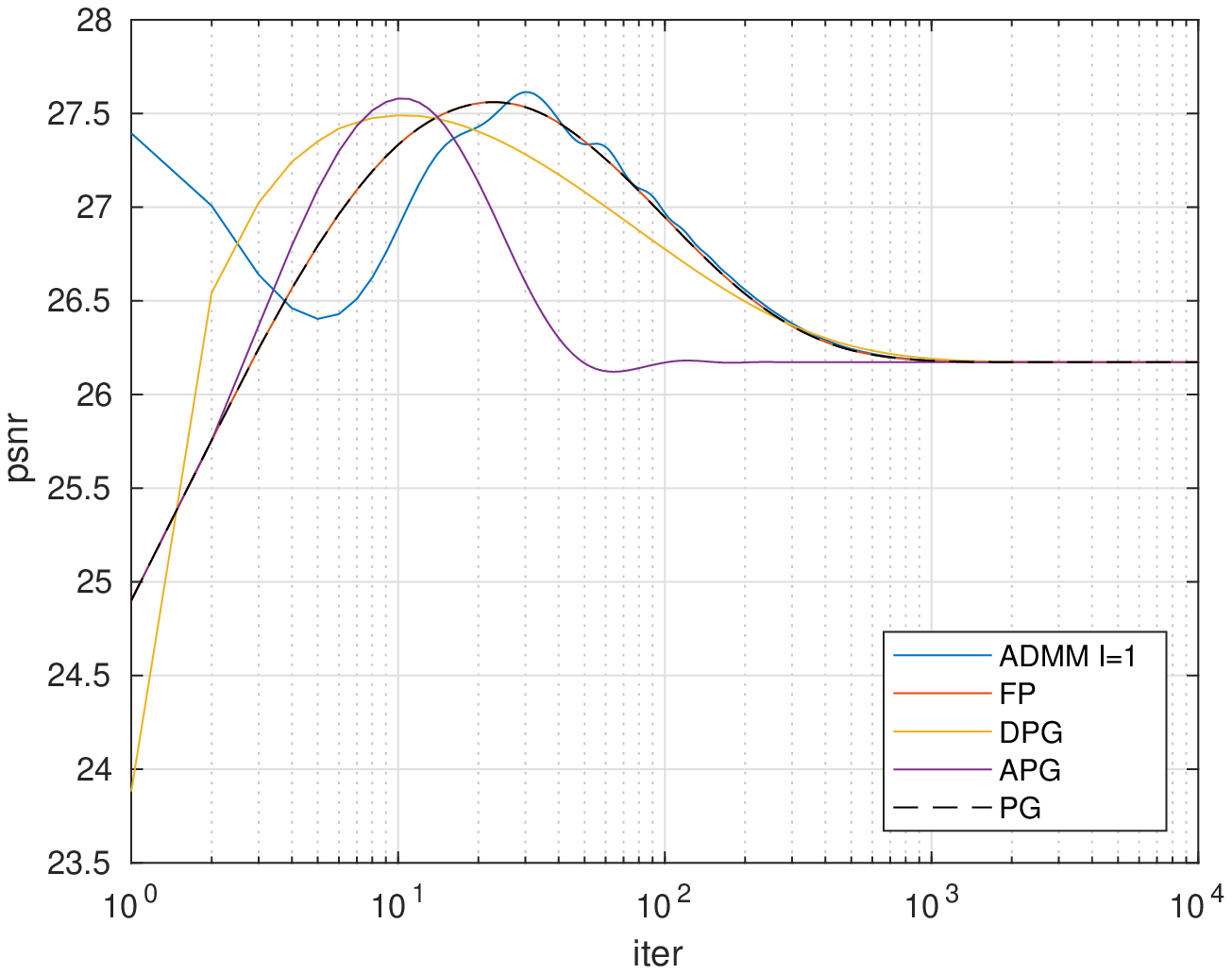}
 	{PSNR versus iteration for RED algorithms with TDT denoising when deblurring the starfish.}
 	{3.4}
 	 {\newcommand{\sz}{0.6}
 	 \psfrag{iter}[t][t][0.7]{\sf iteration}
	 \psfrag{ADMM I=1}[l][l][\sz]{\sf ADMM-$I\!=\!1$}
	 \psfrag{FP}[l][l][\sz]{\sf FP}
	 \psfrag{PG-DP}[l][l][\sz]{\sf DPG}
	 \psfrag{DPG}[l][l][\sz]{\sf DPG}
	 \psfrag{PG}[l][l][\sz]{\sf PG}
	 \psfrag{APG}[l][l][\sz]{\sf APG}
         \psfrag{psnr}[B][B][0.7]{\sf PSNR}}
     	{trim=20pt -5pt 20pt 20pt}
     	
\putFrag{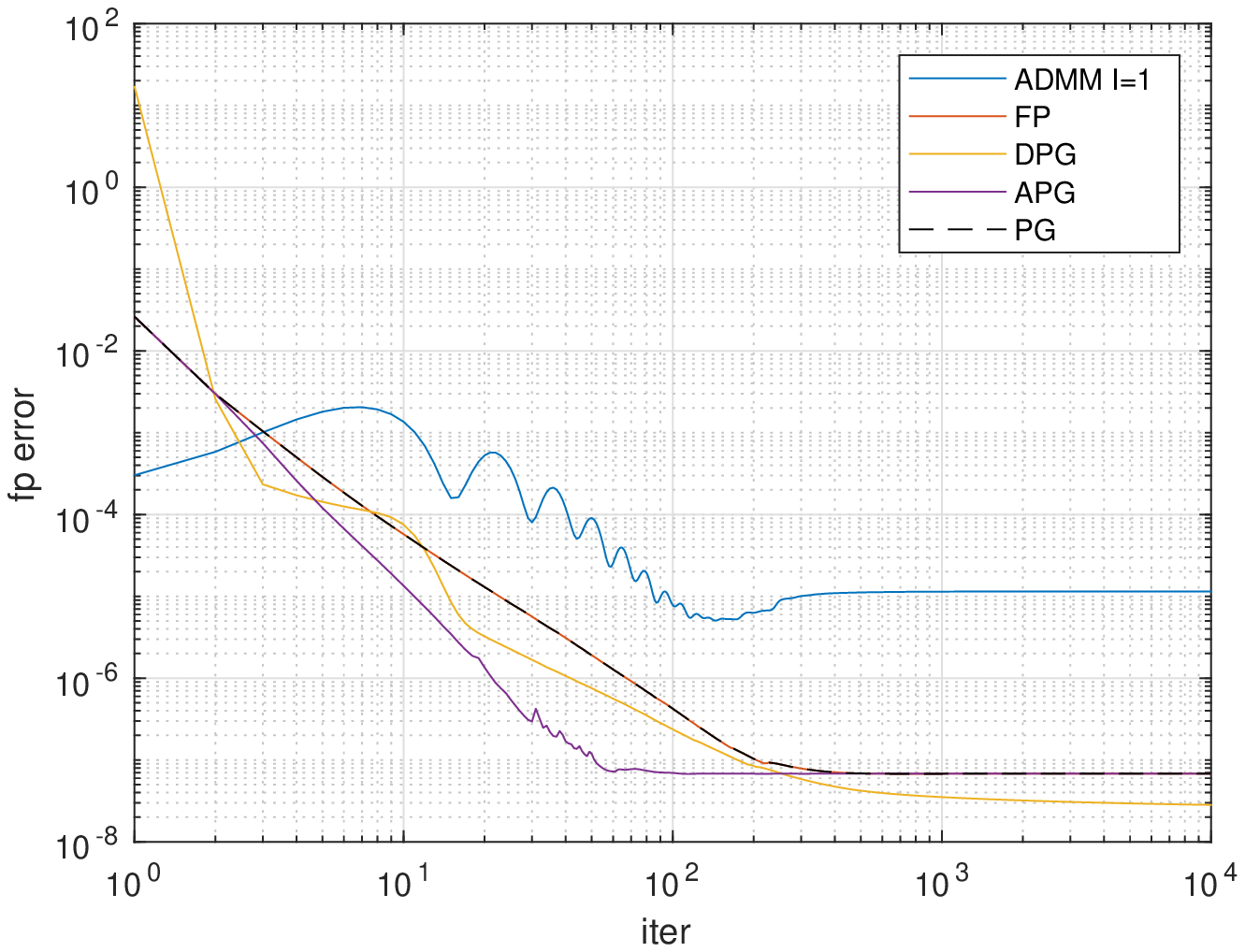}
 	{Fixed-point error versus iteration for RED algorithms with TDT denoising when deblurring the starfish.}
 	{\figsize}
 	 {\newcommand{\sz}{0.6}
 	 \psfrag{iter}[t][t][0.7]{\sf iteration}
	 \psfrag{ADMM I=1}[l][l][\sz]{\sf ADMM-$I\!=\!1$}
	 \psfrag{FP}[l][l][\sz]{\sf FP}
	 \psfrag{PG-DP}[l][l][\sz]{\sf DPG}
	 \psfrag{DPG}[l][l][\sz]{\sf DPG}
	 \psfrag{PG}[l][l][\sz]{\sf PG}
	 \psfrag{APG}[l][l][\sz]{\sf APG}
         \psfrag{fp error}[B][B][0.7]{$\tfrac{1}{N}\big\|\tfrac{1}{\sigma^2}\vec{A}^H(\vec{Ax}_k-\vec{y}) + \lambda(\vec{x}_k-\vec{f}(\vec{x}_k))\big\|^2$}}
     	{trim=20pt 10pt 20pt 20pt}

\putFrag{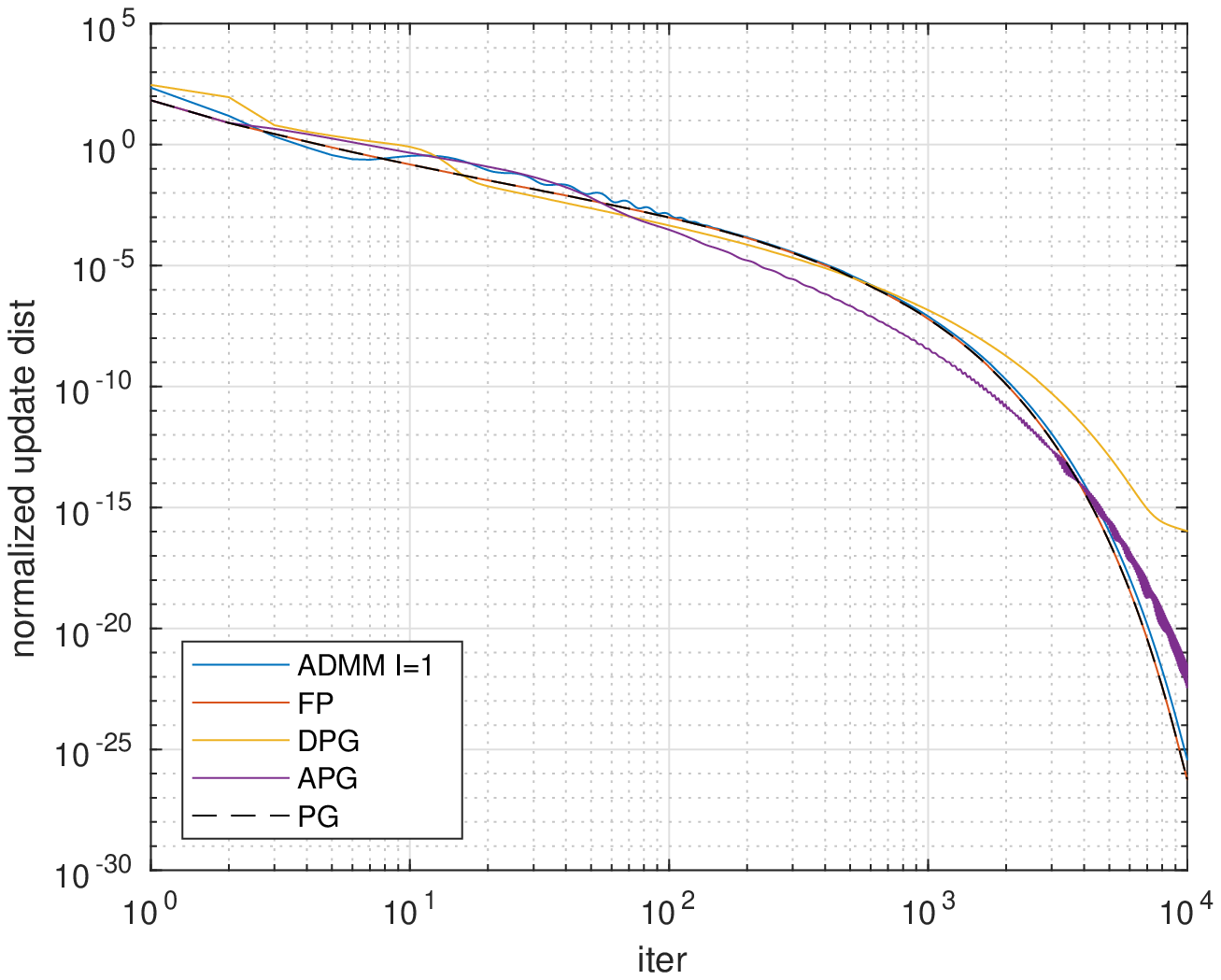}
 	{Update distance versus iteration for RED algorithms with TDT denoising when deblurring the starfish.}
 	{\figsize}
 	 {\newcommand{\sz}{0.6}
 	 \psfrag{iter}[t][t][0.7]{\sf iteration}
	 \psfrag{ADMM I=1}[l][l][\sz]{\sf ADMM-$I\!=\!1$}
	 \psfrag{FP}[l][l][\sz]{\sf FP}
	 \psfrag{PG-DP}[l][l][\sz]{\sf DPG}
	 \psfrag{DPG}[l][l][\sz]{\sf DPG}
	 \psfrag{PG}[l][l][\sz]{\sf PG}
	 \psfrag{APG}[l][l][\sz]{\sf APG}
     \psfrag{normalized update dist}[B][B][0.7]{\sf $\tfrac{1}{N}\|\vec{x}_k-\vec{x}_{k-1}\|^2$}}
     	{trim=20pt 10pt 20pt 20pt}

Next, we replace the TNRD denoiser with the TDT denoiser from \eqref{fTD} and repeat the previous experiments.
For the TDT denoiser, we used a Haar-wavelet based orthogonal discrete wavelet transform (DWT) $\vec{W}$, with the maximum number of decomposition levels, and a soft-thresholding function $\vec{g}(\cdot)$ with threshold value $0.001$.
Unlike the TNRD denoiser, this TDT denoiser is the proximal operator associated with a convex cost function, and so we know that it is $\frac{1}{2}$-averaged and non-expansive.

\Figref{psnr_dwt} shows PSNR versus iteration with TDT denoising.
Interestingly, the final PSNR values appear to be nearly identical among all algorithms under test, but more than $1$~dB worse than the values around iteration $20$.
\Figref{fp_dwt} shows the fixed-point error vs.\ iteration for this experiment.
There, the errors of most algorithms converge to a value near $10^{-7}$, but then remain at that value.
Noting that RED-PG satisfies the conditions of \thmref{mann} (i.e., convex loss, non-expansive denoiser, $L>1$), it should converge to a fixed-point of \eqref{fpRED}.
Therefore, we attribute the fixed-point error saturation in \figref{fp_dwt} to issues with numerical precision.
\Figref{dist_dwt} shows the normalized distance versus iteration with TDT denoising.
There, the distance decreases to zero for all algorithms under test.

We emphasize that the 
proposed RED-DPG, RED-APG, and RED-PG algorithms seek to solve exactly the same fixed-point equation \eqref{fpRED} sought by the RED-SD, RED-ADMM, and RED-FP algorithms proposed in \cite{Romano:JIS:17}. 
The excellent quality of the RED fixed-points was firmly established in \cite{Romano:JIS:17},
both qualitatively and quantitatively,
in comparison to existing state-of-the-art methods like PnP-ADMM \cite{Venkatakrishnan:GSIP:13}.
For further details on these comparisons, including examples of images recovered by the RED algorithms, we refer the interested reader to \cite{Romano:JIS:17}.

\section{Equilibrium View of RED Algorithms} \label{sec:CE}

Like the RED algorithms, PnP-ADMM \cite{Venkatakrishnan:GSIP:13} repeatedly calls a denoiser $\vec{f}(\cdot)$ in order to solve an inverse problem.
In \cite{Buzzard:JIS:18}, Buzzard, Sreehari, and Bouman show that PnP-ADMM finds a ``consensus equilibrium'' solution rather than a minimum of any explicit cost function.
By consensus equilibrium, we mean a solution $(\hvec{x},\hvec{u})$ to 
\begin{subequations}\label{eq:consensus}
\begin{align}
\hvec{x} = F(\hvec{x}+\hvec{u}) 
\label{eq:F} \\
\hvec{x} = G(\hvec{x}-\hvec{u}) 
\label{eq:G} 
\end{align}
\end{subequations}
for some functions $F,G:\Real^N\rightarrow\Real^N$.
For PnP-ADMM, these functions are \cite{Buzzard:JIS:18}
\begin{align}
F\pnp(\vec{v})
&= \arg\min_{\vec{x}} \left\{\ell(\vec{x};\vec{y}) + \frac{\beta}{2}\|\vec{x}-\vec{v}\|^2\right\} 
\label{eq:Fpnp} \\
G\pnp(\vec{v})
&= \vec{f}(\vec{v})
\label{eq:Gpnp} .
\end{align}

\subsection{RED Equilibrium Conditions}

We now show that the RED algorithms also find consensus equilibrium solutions, but with $G\neq G\pnp$.
First, recall ADMM \algref{ADMM} with explicit regularization $\rho(\cdot)$. 
By taking iteration $k\rightarrow\infty$, it becomes clear that the ADMM solutions must satisfy the equilibrium condition \eqref{consensus} with 
\begin{align}
F\admm(\vec{v})
&= \arg\min_{\vec{x}} \left\{\ell(\vec{x};\vec{y}) + \frac{\beta}{2}\|\vec{x}-\vec{v}\|^2\right\} 
\label{eq:Fadmm} \\
G\admm(\vec{v})
&= \arg\min_{\vec{x}} \left\{\lambda\rho(\vec{x}) + \frac{\beta}{2}\|\vec{x}-\vec{v}\|^2\right\} 
\label{eq:Gadmm} ,
\end{align}
where we note that $F\admm=F\pnp$.

The RED-ADMM algorithm can be considered as a special case of ADMM \algref{ADMM} under which
$\rho(\cdot)$ is differentiable with $\nabla\rho(\vec{x})=\vec{x}-\vec{f}(\vec{x})$, for a given denoiser $\vec{f}(\cdot)$. 
We can thus find $G\redadmm(\cdot)$, i.e., the RED-ADMM version of $G(\cdot)$ satisfying the equilibrium condition \eqref{G}, by solving the right side of \eqref{Gadmm} under  $\nabla\rho(\vec{x})=\vec{x}-\vec{f}(\vec{x})$.
Similarly, we see that the RED-ADMM version of $F(\cdot)$ is identical to the ADMM version of $F(\cdot)$ from \eqref{Fadmm}.
Now, the $\hvec{x}=G\redadmm(\vec{v})$ that solves the right side of \eqref{Gadmm} under differentiable $\rho(\cdot)$ with $\nabla\rho(\vec{x})=\vec{x}-\vec{f}(\vec{x})$ must obey
\begin{align}
\vec{0} 
&= \lambda \nabla \rho(\hvec{x}) + \beta(\hvec{x}-\vec{v}) \\
&= \lambda \big(\hvec{x}-\vec{f}(\hvec{x})\big)  + \beta(\hvec{x}-\vec{v}) 
\label{eq:denoiser_fp} ,
\end{align}
which we note is a special case of \eqref{fpRED}.
Continuing, we find that
\begin{align}
\vec{0}
&= \lambda \big(\hvec{x}-\vec{f}(\hvec{x})\big)  + \beta(\hvec{x}-\vec{v}) \\
\Leftrightarrow
\vec{0} 
&= \frac{\lambda+\beta}{\beta}\hvec{x} - \frac{\lambda}{\beta} \vec{f}(\hvec{x}) -\vec{v} \\
\Leftrightarrow
\vec{v} 
&= \left(\frac{\lambda+\beta}{\beta}\vec{I} - \frac{\lambda}{\beta} \vec{f}\right)(\hvec{x}) \\
\Leftrightarrow
\hvec{x} 
&= \left( \frac{\lambda+\beta}{\beta}\vec{I}-\frac{\lambda}{\beta}\vec{f} \right)^{-1}(\vec{v}) 
= G\redadmm(\vec{v}) 
\label{eq:G_RED_ADMM},
\end{align}
where $\vec{I}$ represents the identity operator and $(\cdot)^{-1}$ represents the functional inverse.
In summary, RED-ADMM with denoiser $\vec{f}(\cdot)$ solves the consensus equilibrium problem \eqref{consensus} with $F=F\admm$ from \eqref{Fadmm} and $G=G\redadmm$ from \eqref{G_RED_ADMM}.

Next we establish an equilibrium result for RED-PG.
Defining $\vec{u}_k=\vec{v}_k-\vec{x}_k$
and taking $k\rightarrow\infty$ in \algref{RED_PG}, it can be seen that the fixed points of RED-PG obey \eqref{F} for
\begin{align}
F\redpg(\vec{v})
&= \arg\min_{\vec{x}} \left\{\ell(\vec{x};\vec{y}) + \frac{\lambda L}{2}\|\vec{x}-\vec{v}\|^2\right\} 
\label{eq:Fredpg} .
\end{align}
Furthermore, from \lineref{PG_v_update} of \algref{RED_PG}, it can be seen that the RED-PG fixed points also obey
\begin{align}
\hvec{u}
&= \frac{1}{L}\left( \vec{f}(\hvec{x}) - \hvec{x} \right) 
\label{eq:ured}\\
\Leftrightarrow
\hvec{x}-\hvec{u}
&= \hvec{x} - \frac{1}{L}\left( \vec{f}(\hvec{x})-\hvec{x} \right) \\
&= \left(\frac{L+1}{L}\vec{I} - \frac{1}{L} \vec{f}\right)(\hvec{x}) \\
\Leftrightarrow
\hvec{x}
&= \left(\frac{L+1}{L}\vec{I} - \frac{1}{L} \vec{f}\right)^{-1}(\hvec{x}-\hvec{u}) ,
\end{align}
which matches \eqref{G} when $G=G\redpg$ for
\begin{align}
G\redpg(\vec{v})
&= \left(\frac{L+1}{L}\vec{I} - \frac{1}{L} \vec{f}\right)^{-1}(\vec{v}) 
\label{eq:Gredpg} .
\end{align}
Note that $G\redpg=G\redadmm$ when $L=\beta/\lambda$.

\subsection{Interpreting the RED Equilibria}

The equilibrium conditions provide additional interpretations of the RED algorithms.
To see how, first recall that the RED equilibrium $(\hvec{x},\hvec{u})$ satisfies
\begin{subequations}
\begin{align}
\hvec{x} &= F\redpg(\hvec{x}+\hvec{u}) 
\label{eq:xequilibRED}\\
\hvec{x} &= G\redpg(\hvec{x}-\hvec{u}) ,
\end{align}
\end{subequations}
or an analogous pair of equations involving $F\redadmm$ and $G\redadmm$.
Thus, from \eqref{Fredpg}, \eqref{ured}, and \eqref{xequilibRED}, we have that 
\begin{align}
\hvec{x} 
&= F\redpg\left(\hvec{x}+\frac{1}{L}(\vec{f}(\hvec{x})-\hvec{x})\right) \\
&= F\redpg\left(\frac{L-1}{L}\hvec{x}+\frac{1}{L}\vec{f}(\hvec{x})\right) \\
&= \arg\min_{\vec{x}} \left\{\ell(\vec{x};\vec{y}) + \frac{\lambda L}{2}\left\|\vec{x}-\frac{L-1}{L}\hvec{x}-\frac{1}{L}\vec{f}(\hvec{x})\right\|^2\right\} .
\end{align}
When $L=1$, this simplifies down to
\begin{align}
\hvec{x} 
&= \arg\min_{\vec{x}} \left\{\ell(\vec{x};\vec{y}) + \frac{\lambda}{2}\left\|\vec{x}-\vec{f}(\hvec{x})\right\|^2\right\} 
\label{eq:xequilibRED2}.
\end{align}
Note that \eqref{xequilibRED2} is reminiscent of, although in general not equivalent to,
\begin{align}
\hvec{x} 
&= \arg\min_{\vec{x}} \left\{\ell(\vec{x};\vec{y}) + \frac{\lambda}{2}\left\|\vec{x}-\vec{f}(\vec{x})\right\|^2\right\} ,
\end{align}
which was discussed as an ``alternative'' formulation of RED in \cite[Sec.~5.2]{Romano:JIS:17}.

Insights into the relationship between RED and PnP-ADMM can be obtained by focusing on the simple case of 
\begin{align}
\ell(\vec{x};\vec{y})
&= \frac{1}{2\sigma^2}\|\vec{x}-\vec{y}\|^2 
\label{eq:simple} ,
\end{align}
where the overall goal of variational image recovery would be the denoising of $\vec{y}$.
For PnP-ADMM, \eqref{RED_FPQ} and \eqref{Fpnp} imply
\begin{align}
F\pnp(\vec{v})
&= \frac{1}{1+\lambda\sigma^2}\vec{y} 
   + \frac{\lambda\sigma^2}{1+\lambda\sigma^2}\vec{v} ,
\end{align}
and so the equilibrium condition \eqref{F} implies
\begin{align}
\hvec{x}\pnp
&= \frac{1}{1+\lambda\sigma^2}\vec{y} 
   + \frac{\lambda\sigma^2}{1+\lambda\sigma^2}(\hvec{x}\pnp+\hvec{u}\pnp) \\
\Leftrightarrow
\hvec{u}\pnp
&= \frac{\hvec{x}\pnp-\vec{y}}{\lambda\sigma^2} .
\end{align}
Meanwhile, \eqref{Gpnp} and the equilibrium condition \eqref{G} imply
\begin{align}
\hvec{x}\pnp
&= \vec{f}(\hvec{x}\pnp-\hvec{u}\pnp) \\
&= \vec{f}\left(\frac{\lambda\sigma^2-1}{\lambda\sigma^2}\hvec{x}\pnp
                +\frac{1}{\lambda\sigma^2}\vec{y}\right) .
\end{align}
In the case that $\lambda=1/\sigma^2$, we have the intuitive result that
\begin{align}
\hvec{x}\pnp
&= \vec{f}(\vec{y}) 
\label{eq:equilib_pnp},
\end{align}
which corresponds to direct denoising of $\vec{y}$.
For RED, $\hvec{u}\RED$ is algorithm dependent, but $\hvec{x}\RED$ is always the solution to \eqref{fpRED}, where now $\vec{A}=\vec{I}$ due to \eqref{simple}.
That is,
\begin{align}
\vec{y} -\hvec{x}\RED
&= \lambda\sigma^2\big(\hvec{x}\RED - \vec{f}(\hvec{x}\RED)\big) .
\end{align}
Taking $\lambda=1/\sigma^2$ for direct comparison to \eqref{equilib_pnp}, we find
\begin{align}
\vec{y} -\hvec{x}\RED
&= \hvec{x}\RED - \vec{f}(\hvec{x}\RED) 
\label{eq:equilib_red} .
\end{align}
Thus, whereas PnP-ADMM reports the denoiser output $\vec{f}(\vec{y})$,
\emph{RED reports the $\hvec{x}$ for which the denoiser residual $\vec{f}(\hvec{x}) - \hvec{x}$ negates the measurement residual $\vec{y} - \hvec{x}$.}
This $\hvec{x}$ can be expressed concisely as
\begin{align}
\hvec{x} 
= (2\vec{I}-\vec{f})^{-1}(\vec{y})
= G\redpg(\vec{y})\big|_{L=1} .
\end{align}

\section{Conclusion} \label{sec:conclusion}

The RED paper \cite{Romano:JIS:17} proposed a powerful new way to exploit plug-in denoisers when solving imaging inverse-problems.
In fact, experiments in \cite{Romano:JIS:17} suggest that the RED algorithms are state-of-the-art.
Although \cite{Romano:JIS:17} claimed that the RED algorithms minimize an optimization objective containing an explicit regularizer of the form $\rhoRED(\vec{x})\defn\frac{1}{2}\vec{x}\tran(\vec{x}-\vec{f}(\vec{x}))$ when the denoiser is LH, we showed that the denoiser must also be Jacobian symmetric for this explanation to hold.
We then provided extensive numerical evidence that practical denoisers like
the median filter, 
non-local means,
BM3D,
TNRD,
or DnCNN
lack sufficient Jacobian symmetry.
Furthermore, we established that, with non-JS denoisers, the RED algorithms cannot be explained by explicit regularization of any form.

None of our negative results dispute the fact that the RED algorithms work very well in practice.
But they do motivate the need for a better understanding of RED.
In response, we showed that the RED algorithms can be explained by a novel framework called \emph{score-matching by denoising} (SMD), which aims to match the ``score'' (i.e., the gradient of the log-prior) rather than design any explicit regularizer.
We then established tight connections between SMD, kernel density estimation, and constrained MMSE denoising.

On the algorithmic front, 
we provided new interpretations of the RED-ADMM and RED-FP algorithms proposed in \cite{Romano:JIS:17}, and we proposed novel RED algorithms with much faster convergence.
Finally, we performed a consensus-equilibrium analysis of the RED algorithms that lead to additional interpretations of RED and its relation to PnP-ADMM.

\section*{Acknowledgments}

The authors thank Peyman Milanfar, Miki Elad, Greg Buzzard, and Charlie Bouman  for insightful discussions.


\bibliographystyle{ieeetr}
\bibliography{macros_abbrev,books,misc,machine,sparse,phase} 

\end{document}